\documentclass[11pt]{article}
\usepackage[margin = 1in]{geometry}

\usepackage{lipsum}
\usepackage{amsfonts}
\usepackage{graphicx}
\usepackage{epstopdf}
\usepackage[font=small]{caption}

\usepackage{mathtools}
\usepackage{mathrsfs}
\usepackage{empheq}
\usepackage{amsfonts}
\usepackage{bbm,amsgen,amsthm,amsmath,amstext,amsbsy,amsopn,amssymb,subfigure,stmaryrd}
\usepackage{comment}
\usepackage{enumitem}
\usepackage{natbib}
\usepackage{booktabs}
\usepackage{blkarray}
\usepackage{algorithm}
\usepackage{algpseudocode}
\usepackage{url}
\usepackage{longtable}
\usepackage{mathtools}
\usepackage{multirow}
\usepackage{xcolor}
\usepackage{dsfont}
\usepackage{tikz}
\usetikzlibrary{shapes.multipart}
\usetikzlibrary{automata,positioning}

\usepackage[
  colorlinks,
  citecolor=blue,
  linkcolor=red,
  anchorcolor=red,
  urlcolor=blue
]{hyperref}

\ifpdf
  \DeclareGraphicsExtensions{.eps,.pdf,.png,.jpg}
\else
  \DeclareGraphicsExtensions{.eps}
\fi

\usepackage{tikz}
\usetikzlibrary{decorations.pathreplacing,calligraphy}

\newtheorem{Definition}{Definition}
\newtheorem{Theorem}{Theorem}

\newtheorem{Lemma}{Lemma}

\newtheorem{Proposition}{Proposition}

\newcommand{\R}{\mathbb{R}}

\newcommand{\EE}[1]{\mathbb{E}\left[{#1}\right]}

\newcommand{\PP}[1]{\mathbb{P}\left\{{#1}\right\}}
\newcommand{\PPst}[2]{\mathbb{P}\left\{{#1}\  \middle| \ {#2}\right\}}

\newcommand{\One}[1]{{\mathbbm{1}}\left\{{#1}\right\}}

\newcommand{\iidsim}{\stackrel{\textnormal{iid}}{\sim}}

\newcommand{\bbR}{\mathbb{R}}

\newcommand{\cA}{\mathcal{A}}

\newcommand{\cD}{\mathcal{D}}
\newcommand{\cY}{\mathcal{Y}}
\newcommand{\cX}{\mathcal{X}}
\newcommand{\cE}{\mathcal{E}}

\newcommand{\hT}{{\widehat{T}}}
\newcommand{\hr}{{\widehat{r}}}
\newcommand{\hf}{{\widehat{f}}}
\newcommand{\cZ}{\mathcal{Z}}

\newcommand{\hy}{{\widehat{y}}}
\newcommand{\indi}{{\mathds{1}}}

\newcommand{\bbP}{\mathbb{P}}
\newcommand{\bbE}{\mathbb{E}}

\allowdisplaybreaks

\graphicspath{ {images/} }

\begin{document}
	\title{Is Algorithmic Stability Testable? A Unified Framework under Computational Constraints}
\author{Yuetian Luo\thanks{Data Science Institute, University of Chicago} \ and Rina Foygel Barber\thanks{Department of Statistics, University of Chicago}}
	\date{}
	\maketitle

	\bigskip

\begin{abstract}
Algorithmic stability is a central notion in learning theory that quantifies the sensitivity of an algorithm to small changes in the training data. If a learning algorithm satisfies certain stability properties, this leads to many important downstream implications, such as generalization, robustness, and reliable predictive inference. Verifying that stability holds for a particular algorithm is therefore an important and practical question. However, recent results establish that testing the stability of a black-box algorithm is impossible, given limited data from an unknown distribution, in settings where the data lies in an uncountably infinite space (such as real-valued data). In this work, we extend this question to examine a far broader range of settings, where the data may lie in any space---for example, categorical data. We develop a unified framework for quantifying the hardness of testing algorithmic stability, which establishes that across all settings, if the available data is limited then exhaustive search is essentially the only universally valid mechanism for certifying algorithmic stability. Since in practice, any test of stability would naturally be subject to computational constraints, exhaustive search is impossible and so this implies fundamental limits on our ability to test the stability property for a black-box algorithm.
\end{abstract}

\section{Introduction}\label{sec:intro}
For data lying in space $\cX \times \cY$, an algorithm is a map that maps data sets to fitted regression functions:
\begin{equation} \label{def:deter-alg}
	\cA: \  \cup_{n\geq 0} (\cX\times\cY)^n \rightarrow \big\{\textnormal{measurable functions $\cX\rightarrow \bbR$}\big\}.
\end{equation}
That is, $\cA$ takes as input $\cD=\{(X_i,Y_i)\}_{i \in [n]}$, a data set of any size $n$, and returns a fitted model $\hf_n: \cX\rightarrow\bbR$,
which maps a feature value $x$ to a fitted value, or prediction, $\hy = \hf_n(x)$. \footnote{Throughout this paper, we implicitly assume measurability for all
functions we define, as appropriate. For the algorithm $\cA$,
we assume measurability of the map
$\left( (x_1, y_1), \ldots, (x_n, y_n), x \right) \mapsto [\cA(\{(x_i,y_i)\}_{i\in[n]})](x)$.} {\it Algorithmic stability} is a notion that quantifies the sensitivity of an algorithm to small changes in the training data $\cD$. Since its introduction \citep{devroye1979distribution,rogers1978finite,kearns1997algorithmic}, algorithmic stability has been a central notion in machine learning and statistics. Intuitively,
it is a desirable property for learning algorithms---but it also leads to many downstream implications, as well.
For instance, it is known that stable algorithms generalize \citep{kearns1997algorithmic,bousquet2002stability,poggio2004general,shalev2010learnability}. It also plays a critical role in a number of different problems, such as model selection \citep{meinshausen2010stability,shah2013variable,ren2023derandomizing}, reproducibility \citep{yu2013stability}, predictive inference \citep{steinberger2018conditional,barber2021predictive,liang2023algorithmic}, selective inference \citep{zrnic2023post} and inference for cross-validation error \citep{austern2020asymptotics,bayle2020cross,kissel2022high}. 

Thus, verifying the stability of a given algorithm is an important and practical question. Some algorithms can easily be shown to be stable, such as $k$-nearest neighbors (kNN) \citep{devroye1979distribution} and ridge regression \citep{bousquet2002stability}. More recently, substantial progress has been made towards understanding the stability for more sophisticated algorithms, see \cite{hardt2016train,lei2020fine,deng2021toward} and references therein. However, many complex modern machine learning algorithms exhibit instability empirically, 
or even if this is not the case, they may be too complex for a theoretical stability guarantee to be possible to obtain. 
It is common to attempt to stabilize a black-box algorithm by using generic methods such as bootstrap, bagging, or subsampling \citep{efron1994introduction,breiman1996bagging,elisseeff2005stability,chen2022debiased,soloff2023bagging} so that stability is ensured---but these methods also have their limitations, such as loss of statistical efficiency due to subsampling.
Consequently, this motivates the question of whether we can estimate the stability of an algorithm through empirical evaluation, rather than theoretical calculations. This idea is formalized in the work \cite{kim2021black} where the authors consider inferring the stability of an algorithm through ``{\it black-box}'' testing without peering inside the construction of the algorithm. Under the setting when $\cX$ and $\cY$ spaces have uncountably infinitely many elements and the black-box test can call $\cA$ countably infinitely many times, \cite{kim2021black} established a fundamental bound on our ability to test for algorithmic stability in the absence of any assumptions on the data distribution and the algorithm $\cA$.

However, in many common settings, such as categorical features or a binary response, we may have \emph{finite} data spaces $\cX$ and/or $\cY$.
The results of \cite{kim2021black}, which assume $\cX$ and/or $\cY$ to be infinite, cannot be applied to characterize the hardness of testing stability in these settings---and indeed, we would expect the testing problem to be fundamentally easier if there are only finitely many possible values of the data. On the other hand, in 
practice, we only have a finite amount of computational resources in order to
run our test of stability.  This leads to a key motivating question: is it possible to construct a black-box test
of algorithmic stability in finite data spaces $\cX$ and $\cY$, and what role do computational constraints play in this question? 

\subsection{Overview of contributions}
In this paper, we extend the results in \cite{kim2021black} to a far broader range of settings and establish fundamental bounds for testing algorithmic stability with computational constraints. Specifically, we provide an explicit upper bound on the power of any universally valid black-box test for testing algorithmic stability in terms of the number of available data, stability parameters, computational budget, and the sizes of the data spaces  $\cX$ and $\cY$.

\paragraph{Organization.}
In Section~\ref{sec:background}, we will present some background 
on algorithmic stability and black-box testing framework under computational constraints.
Section~\ref{sec:main-results} presents our main results on the limit of black-box tests for algorithmic stability with computational constraints. A brief discussion is provided in Section~\ref{sec:discussion}. We prove the main results in Section \ref{sec:main-proof}. Extensions and additional proofs are deferred to the Appendix.

\section{Background and problem setup} \label{sec:background}

In this section, we formalize some definitions and give additional background for studying algorithmic stability via black-box tests. In \eqref{def:deter-alg}, we have defined the algorithm as a deterministic map given the training data. In many settings, commonly used algorithms might also include some form of randomization---for instance, stochastic gradient
descent type methods. To accommodate this, we will expand the definition of an algorithm to take an additional argument,
\begin{equation}\label{eqn:define_alg}\cA:  \ \cup_{n\geq 0} (\cX\times\cY)^n \times [0,1] \rightarrow \big\{\textnormal{measurable functions $\cX\rightarrow\bbR$}\big\}.\end{equation}
The fitted model $\hf_n = \cA\big(\{(X_i,Y_i)\}_{i \in [n]} ; \xi\big)$ is now obtained by running algorithm $\cA$
on data set $\{(X_i,Y_i)\}_{i \in [n]}$, with randomization provided by the argument $\xi\in[0,1]$, 
which acts as a random seed. Of course, the general definition~\eqref{eqn:define_alg}
of a randomized algorithm also accommodates the deterministic (i.e., non-random) setting and we refer any algorithm $\cA$ that depends only on the data set $\{(X_i,Y_i)\}_{i\in [n]}$, and ignores the second argument $\xi$ as deterministic algorithm. In this paper, we focus on the setting where $\cA$ is symmetric in the training data---that is, permuting the order of the $n$ data points does not change the output of $\cA$, i.e.,
\begin{equation*}
	\cA[(x_1, y_1), \ldots, (x_n, y_n); \xi ] = \cA[(x_{\sigma(1)}, y_{\sigma(1)}), \ldots, (x_{\sigma(n)}, y_{\sigma(n)}); \xi' ],
\end{equation*} for some coupling of the randomization terms $\xi, \xi'  \sim \textnormal{Unif}[0,1]$.

\subsection{Algorithmic stability and hypothesis testing formulation}
There are many different notions of algorithmic stability in the literature (see, e.g., \cite{bousquet2002stability} and Appendix A of \cite{shalev2010learnability} for some comparisons). In this work, we follow \cite{kim2021black} and consider testing the following notion of stability: 
  \begin{Definition}[Algorithmic stability] \label{def:stability}
 Let $\cA$ be a symmetric algorithm.
 Let $\epsilon \geq 0$ and $\delta \in [0,1)$. We say that $\cA$ is $(\epsilon, \delta)$-stable with respect to training data sets of size $n$ from a data distribution $P$---or, for short, the triple $(\cA, P, n)$ is $(\epsilon, \delta )$-stable---if 
 \begin{equation*}
 	\bbP \{ |\hf_n(X_{n+1}) - \hf_{n-1}(X_{n+1}) | > \epsilon \} \leq \delta,
 \end{equation*} where 
 \[\hf_n = \cA\left(\{(X_i,Y_i)\}_{i\in[n]}; \xi\right), \quad \hf_{n-1} = \cA\left(\{(X_i,Y_i)\}_{i\in[n-1]}; \xi\right),\]
and where the data is distributed as $(X_i, Y_i) \iidsim P$, and the random seed $\xi \sim \textnormal{Unif}[0,1]$ is drawn independently of the data.
 \end{Definition} 
In the setting of predictive inference for regression problems, this notion of stability has appeared in \cite{barber2021predictive} where it was shown that this condition is sufficient for predictive coverage guarantees of the jackknife and jackknife+ methods. Earlier work by \cite{steinberger2018conditional} also established coverage properties for jackknife under a similar stability property. 

\paragraph{Defining a hypothesis test.} Given a fixed $\epsilon \geq 0$ and $\delta \in [0,1)$, we are interested in testing
\begin{equation*}
	H_0: (\cA,P,n) \textnormal{ is not } (\epsilon, \delta )\textnormal{-stable} \quad \textnormal{v.s.} \quad H_1: (\cA,P,n) \textnormal{ is } (\epsilon, \delta )\textnormal{-stable}
\end{equation*}
with some desired bound $\alpha \in (0,1)$ on the type-I error---that is, a bound on the probability that we falsely certify an algorithm as stable. In this work, we consider the problem of running this hypothesis test with limited
access to data from $P$---specifically, we assume that the available data consists of a labeled data set, $\cD_{\ell} = \{ (X_i, Y_i)_{i \in [N_{\ell}]} \} \iidsim P $,  and an unlabeled data set,  $\cD_u = \{ X_{N_{\ell}+1}, \cdots, X_{N_{\ell} + N_u} \} \iidsim P_X$, where $P_X$ is the marginal distribution of $X$ under $P$ (and $\cD_{\ell}$ and $\cD_u$ are sampled independently). With the aim of avoiding placing any assumptions on the algorithm $\cA$ and on the distribution $P$, we would like to seek a test $\hT(\cA, \cD_{\ell}, \cD_u) \in \{0,1 \}$ that satisfies the following notion of {\it assumption-free} validity:
 \begin{equation}\label{eqn:validity}\bbP_P\{ \hT(\cA,\cD_{\ell}, \cD_u) = 1 \}\leq \alpha \textnormal{ for any $(\cA,P,n)$ that is not $(\epsilon,\delta)$-stable}.\end{equation}
Although our notation accommodate randomized algorithms, in some cases if we know we are working with deterministic algorithms, we might not need to assume the above strong notion of validity, instead we can require its deterministic conterpart:
 \begin{equation}\label{eqn:validity-deter}\bbP_P\{ \hT(\cA,\cD_{\ell}, \cD_u) = 1 \}\leq \alpha \textnormal{ for any $(\cA,P,n)$ that is not $(\epsilon,\delta)$-stable and } \cA \textnormal{ is deterministic}.\end{equation}

\subsection{Black-box tests with computational constraints} \label{sec:def-black-box-test}

To understand the difficulty of testing algorithmic stability, we need to 
formalize the limitations on what information is available to us, as the analyst, for performing the test.
For example, if we know the distribution $P$ of the data, then for many algorithms, we might be able to study its stability property analytically. Clearly, this is not the setting of interest, since in practice we cannot rely on assumptions about the data distribution---but in fact, even if $P$ is unknown, analyzing the stability for simple algorithms such as kNN or ridge regression is still tractable \citep{devroye1979distribution,bousquet2002stability}. In modern settings, however, a state-of-the-art algorithm $\cA$ is typically far more complex and may be too complicated to analyze theoretically. For this reason, we consider the black-box testing framework introduced in \cite{kim2021black} which analyzes the stability of algorithm $\cA$ via empirical evaluations, rather than theoretical calculations. The intuition is this: if $\cA$ is implemented as a function in some software package, our procedure is allowed to call the function, but is not able to examine the implementation of $\cA$ (i.e., we cannot read the code that defines the function).

 A high-level illustration of a black-box test of interest is provided in Figure \ref{fig:black-box-alg} and the formal definition is provided in Definition \ref{def:black-box-test}. Compared to the definition of black-box test given in \cite{kim2021black}, here we focus on black-box tests with computational constraints. In particular, the computational budget for training is characterized by $B_{\textnormal{train}}$---the total number of data points that we are allowed to input to the algorithm $\cA$, over all calls to $\cA$ (for example, when running the test, we can choose
to train $\cA$ on a data set of size $n$, up to $B_{\textnormal{train}}/n$ many times). 

\begin{figure}[t]
	\centering\medskip
	\fbox{
	\begin{tikzpicture}[
	roundnode/.style={circle, draw=gray!60, fill=gray!5, thick, minimum size=13mm},
	  >=stealth,
    node distance=3cm,
    auto
	]
	\node (left) [fill=blue!25,text width=3.3cm] at (-3.8,0) {Input:\\ labeled data $\cD_{\ell}$,\\ unlabeled data $\cD_{u}$,\\ algorithm $\cA$,\\
	 seeds $\zeta,\zeta_1,\zeta_2,\dots$};
		\draw[-latex, line width=0.5mm] (-2,0) -> (-1,0);
		\draw[black, thick] (-1,-3.3) rectangle (4,2.3);
		\node (upper1) [text width=8cm] at (2,3.) {Iterate for rounds $r = 1,2,\ldots,\hr$, subject to\\ computational constraint  $\sum_{r=1}^{\hr}|\cD^{(r)}_\ell|\leq B_{\textnormal{train}}$};
			\node (upper) [fill=blue!25,text width=4.5cm] at (1.5,1.3) {Generate labeled data set $\cD^{(r)}_{\ell}$ and random seed $\xi^{(r)}$};
		\node (down) [fill=blue!25,text width=4.5cm] at (1.5,-2.1) {Call the algorithm to train a model:\\ $\hf^{(r)} = \cA(\cD^{(r)}_{\ell};\xi^{(r)})$ };
		\draw[ultra thick,->] (0,0.7) .. controls (-0.5,-0.2) .. (0.,-1.2);
		\draw[ultra thick,->] (2.9,-1.2) .. controls (3.4,-0.2) .. (2.9,0.7);
		\draw[-latex, line width=0.5mm] (4.,0) -> (4.8,0);
		\node (right) [fill=blue!25,text width=3.7cm] at (6.8,0) {Output:\\ $\hT(\cA,\cD_{\ell}, \cD_u ) \in \{0,1\}$};
	\end{tikzpicture}
}
	
	\caption{Black-box test with computational constraints. Here in round $r$, the generating process of the new data set $\cD_{\ell}^{(r)}$ and random seed $\xi^{(r)}$ for training can depend on the input data and on all the past information from previous rounds---e.g., we may create a new data set by resampling from past data. The input random seeds $\zeta,\zeta_1,\zeta_2,\dots\iidsim\textnormal{Unif}[0,1]$ may be used to provide randomization if desired. The procedure stops at a data-dependent stopping time $\hr$, with the constraint that the total number of training points used in calls to $\cA$ cannot exceed the budget $B_{\textnormal{train}}$.
	}
	\label{fig:black-box-alg}
\end{figure}
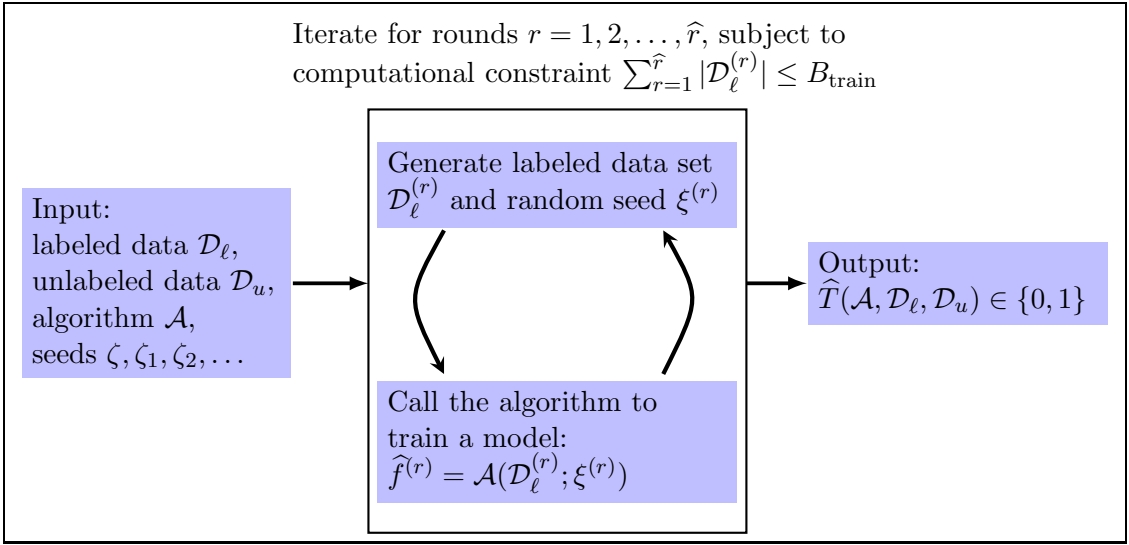

\begin{Definition}[Black-box test under computational constraints] \label{def:black-box-test} Let $\mathfrak{D}_{\ell} = \cup_{m \geq 0} (\cX \times \cY)^m$ denotes the space of labeled data sets of any size, and let $\mathfrak{D}_u = \cup_{m \geq 0} (\cX)^m$ denotes the space of unlabeled data sets of any size.
Consider any function $\hT$ that takes as input an algorithm $\cA$, a labeled data set $\cD_{\ell}\in \mathfrak{D}_{\ell}$, and an unlabeled data set $\cD_u \in \mathfrak{D}_u$,
and returns a (possibly randomized) output $\hT(\cA,\cD_{\ell}, \cD_u)\in\{0,1\}$. 
We say that $\hT$ is a black-box test with computational budget $B_{\textnormal{train}} >0$  for model training if it can be defined as follows:  
 for some functions $g^{(1)},g^{(2)},\dots$ and $g$, and for random seeds, $\zeta^{(1)},\zeta^{(2)},\dots,\zeta\iidsim  \textnormal{Unif}[0,1]$,
 		 \begin{enumerate}[leftmargin=*]
	 	\item For each stage $r = 1,2,3,\ldots$, 
	 	\begin{enumerate}
	 	\item If $r = 1$, generate a new labeled data set and a new random seed,
		\[(\cD_{\ell}^{(1)},\xi^{(1)} ) = g^{(1)} (\cD_{\ell}, \cD_u, \zeta^{(1)} ).\]
		 \begin{itemize}
		 	\item If some stopping criterion is reached or $|\cD^{(1)}_{\ell}| > B_{\textnormal{train}}$, exit the for loop; otherwise compute the fitted model $\hf^{(1)} = \cA( \cD_{\ell}^{(1)}; \xi^{(1)} )$.
		 \end{itemize}
	 	\item If $r \geq 2$, generate a labeled data set and a new random seed,
	 	\begin{equation*}
	 		(\cD_{\ell}^{(r)},\xi^{(r)} ) = g^{(r)} \big(\cD_{\ell}, \cD_u, (\cD_{\ell}^{(s)})_{1\leq s<r}, (\hf^{(s)})_{1\leq s<r}, (\zeta^{(s)})_{1\leq s\leq r} , (\xi^{(s)})_{1\leq s<r}\big).
	 	\end{equation*}
		\begin{itemize}
			\item If some stopping criterion is reached or $\sum_{j = 1}^{r} |\cD_{\ell}^{(j)}| > B_{\textnormal{train}}$, exit the for loop; otherwise compute the fitted model $\hf^{(r)} = \cA( \cD_{\ell}^{(r)}; \xi^{(r)} )$.
		\end{itemize}
	 	\end{enumerate}
		\item Let $\hr$ denote the total number of rounds when stopping.
	 	\item Finally, return
		\[\hT(\cA,\cD_{\ell}, \cD_u ) = g\big(\cD_{\ell}, \cD_u,(\cD_{\ell}^{(r)})_{1\leq r\leq \hr},(\hf^{(r)})_{1\leq r\leq \hr},(\zeta^{(r)})_{1\leq r\leq \hr},(\xi^{(r)})_{1\leq r\leq \hr},\zeta\big).\]
		\end{enumerate}
\end{Definition}
The random seeds $\zeta^{(1)},\zeta^{(2)},\dots$ are included to allow for randomization at each stage of the test, if desired. For example,
if the data set $\cD^{(r)}_{\ell}$ is computed via subsampling from the available real data $\cD_{\ell}$ (e.g., a bootstrap sample),
then $\zeta^{(r)}$ may be used to generate the random subsample. Similarly, the random seed $\zeta$ is included 
to allow for randomization, if desired,
in the final step of the test---i.e., computing $\hT\in\{0,1\}$.

We note that for the black-box test defined in Definition \ref{def:black-box-test}, while it can not examine the theoretical properties of $\cA$, it knows the analytic expression of fitted models $\hf^{(r)}$ (See also Appendix \ref{app:black-box-test-black-box-model} for an alternative type of black-box test). This offers significantv flexibility to the analyst---and thus, establishing the hardness of testing algorithmic stability in this setting is a strong result.

\paragraph{An example of a black-box test.} As an example of the type of test that would satisfy Definition \ref{def:black-box-test}, we can consider a simple sample-splitting based testing procedure. Recall that $N_{\ell}$ and $N_u$ denote the number of labeled and unlabeled data points we have. Define 
\begin{equation} \label{def:kappa}
	\kappa = \kappa(n, B_{\textnormal{train}}, N_{\ell}, N_u) = \min \left\{ \frac{B_{\textnormal{train}}}{n}, \frac{N_{\ell}}{n}, \frac{N_{\ell} + N_u}{n}  \right\}.
\end{equation} Then, $\lfloor \kappa \rfloor$ is the largest number of copies of independent datasets---each consisting of $n$ labeled training points and one unlabeled test point---that can be constructed from $\cD_{\ell}$ and $\cD_u$ and we can run under the computational constraint.
\begin{itemize}
	\item For each $k = 1, \ldots, \lfloor \kappa \rfloor$, train models
	\begin{equation*}
		\begin{split}
			\hf^{(k)}_n &= \cA[(X_{(k-1)n +1}, Y_{(k-1)n + 1}), \ldots, (X_{kn}, Y_{kn}); \xi^{(k)} ],\\
			\hf^{(k)}_{n-1} &= \cA[ (X_{(k-1)n +1}, Y_{(k-1)n + 1}), \ldots, (X_{kn-1}, Y_{kn-1}); \xi^{(k)} ],
		\end{split}
	\end{equation*} with $\xi^{(k)} \overset{i.i.d.}\sim \textnormal{Unif}[0,1]$, and compute the difference in predictions at test point $X_{ \lfloor \kappa \rfloor n + k }$
	\begin{equation} \label{eq:delta-k}
		\Delta^{(k)} = \Big| \hf^{(k)}_n(X_{ \lfloor \kappa \rfloor n + k } ) - \hf^{(k)}_{n-1}(X_{ \lfloor \kappa \rfloor n + k }) \Big|. 
	\end{equation}
	\item Compute $B = \sum_{k=1}^{ \lfloor \kappa \rfloor } \indi( \Delta^{(k)} > \epsilon )$, and compare this test statistic against the Binomial$( \lfloor \kappa \rfloor, \delta )$ distribution, returning $1$ if $B$ is sufficiently small and $0$ otherwise.
\end{itemize}

\section{ Main results on black-box test for algorithmic stability with computational constraints} \label{sec:main-results}
In this section, we present the main results on power upper bounds of testing algorithmic stability for computationally constrained black-box tests. First, given any $(\cA, P,n)$, let us define 
\begin{equation} \label{def:delta_star}
\delta^*_{\epsilon} = \bbP \{ | \hf_n(X_{n+1}) - \hf_{n-1}(X_{n+1}) | > \epsilon  \},
\end{equation} i.e., $\delta^*_{\epsilon}$ is the smallest possible value of $\delta'$ so that $(\cA,P,n)$ is $(\epsilon, \delta')$-stable. 

Next, we state two results regarding the power of testing algorithmic stability for any algorithms or deterministic algorithms.

\begin{Theorem} \label{thm:limits_evaluate-random} Fix any $\epsilon \geq 0$, $\delta \in [0,1 )$. Let $\hT$ be any black-box test with computational budget $B_{\textnormal{train}}$ defined in Definition \ref{def:black-box-test}. Suppose $\hT$ satisfies the assumption-free validity condition \eqref{eqn:validity} at level $\alpha$. Then for any algorithm $\cA$ with $(\cA,P,n)$ that is $(\epsilon,\delta )$-stable (i.e., $\delta^*_{\epsilon} \leq \delta$), the power of $\hT$ is bounded as 
\begin{itemize}
	\item(computational constraint due to budget $B_{\textnormal{train}}$) \begin{equation}\label{ineq:power-bound-comp-rand}
			\bbP \{ \hT = 1 \} \leq \alpha \cdot  \left( \frac{1 - \delta^*_{\epsilon} }{1 - \delta} \right)^{\frac{B_{\textnormal{train} }}{n} }.
\end{equation}. 
	\item(statistical constraint due to limited samples in $\cX$ or $\cY$) 
	\begin{equation}\label{ineq:power-bound}
		\bbP \{ \hT = 1 \} \leq \alpha \cdot \min \Bigg\{  \frac{\left( \frac{1 - \delta^*_{\epsilon} }{1 - \delta} \right)^{ \frac{N_{\ell}}{n} }}{1 - \frac{B_{\textnormal{train}}}{|\cY|} \wedge 1 } , \frac{\left( \frac{1 - \delta^*_{\epsilon} }{1 - \widetilde{\delta} }\right)^{ \frac{N_{\ell} + N_u}{n} }}{1 - \frac{B_{\textnormal{train}}}{|\cX|}\wedge 1 }   \Bigg\},
\end{equation} where $\widetilde{\delta} = (\delta + \frac{1}{n} ) \wedge 1$. 
\end{itemize} 
\end{Theorem}

\begin{Theorem} \label{thm:limits_evaluate-deter} Fix any $\epsilon \geq 0$, $\delta \in [0,1 )$. Let $\hT$ be any black-box test with computational budget $B_{\textnormal{train}}$ defined in Definition \ref{def:black-box-test}. Suppose $\hT$ satisfies the assumption-free validity condition \eqref{eqn:validity-deter} at level $\alpha$ with respect to the class of deterministic algorithms. Then for any algorithm $\cA$ with $(\cA,P,n)$ that is $(\epsilon,\delta )$-stable (i.e., $\delta^*_{\epsilon} \leq \delta$) and $\cA$ is deterministic,
\begin{itemize}
	\item (computational constraint due to budget $B_{\textnormal{train}}$) if the distribution $P$ satisfies 
\begin{equation*}
	\sup_{(x,y) \in \cX \times \cY} \bbP(\{x,y\}) < 0.2,
\end{equation*}
then there is a universal constant $C > 0$ such that the power of $\hT$ is bounded as 
\begin{equation}\label{ineq:power-bound-comp-deter}
		\bbP \{ \hT = 1 \} \leq \left(\alpha + C/n \right)  \cdot \left( \frac{1 - \delta^*_{\epsilon} }{1 - \delta - \frac{1}{n} } \right)^{\frac{B_{\textnormal{train} }}{n} }.
\end{equation}
\item(statistical constraint due to limited samples in $\cX$ or $\cY$) the power of $\hT$ is bounded by \eqref{ineq:power-bound}.
\end{itemize}
\end{Theorem}

From Theorems \ref{thm:limits_evaluate-random} and \ref{thm:limits_evaluate-deter}, we can see the power of any universally valid test is bounded by the minimum of three terms. 
Here we take a closer look at each of the three terms to better understand how this limit on power arises. Let us take the randomized algorithm as an illustration example. 
\begin{itemize}
\item First term (\eqref{ineq:power-bound-comp-rand} or \eqref{ineq:power-bound-comp-deter}): computational constraint due to budget $B_{\textnormal{train}}$. 
The first term is low---that is, not much larger than $\alpha$, meaning that the power cannot be much better than random---whenever the ratio $B_{\textnormal{train}}/n$ is small.
This is intuitive: this ratio describes the number of times we are able to call $\cA$ on training samples of size $n$,
and we need this number to be large in order to distinguish between stability (i.e., large perturbations occur with small probability $\delta^*_\epsilon\leq \delta$)
versus instability (i.e., large perturbations occur with probability $>\delta$).
\item Second term (the first part of the bound in \eqref{ineq:power-bound}): statistical constraint due to limited samples in $\cY$. The second term 
shows that power is low whenever
we have \emph{both} of the following: the ratio $N_\ell/n$ is not too large, and the ratio $B_{\textnormal{train}}/|\cY|$ is close to 0. 
For the first quantity, a bound on $N_\ell/n$ means that we do not have access to ample labeled data---the available data
sampled from $P$ provides only
 $N_\ell/n$ many independent data sets of size $n$, so it is not easy to empirically estimate how often instability may occur.
For the second ratio, if $B_{\textnormal{train}}/|\cY|\approx 0$ then this means that we do not have sufficient computational budget
to allow an exhaustive search over $\cY$ (i.e., we cannot afford to test the behavior of $\cA$ on all possible values of $Y$)
to compensate for the limited amount of available labeled data.
\item Third term (the second part of the bound in \eqref{ineq:power-bound}): statistical constraint due to limited samples in $\cX$. The interpretation for this term is analogous to the previous term,
with the difference that we are given $N_\ell+N_u$ many draws of $X$ due to the additional unlabeled data set $\cD_u$.
\end{itemize}
In particular, the power can be only slightly higher than $\alpha$ (i.e., only slightly better than random),
unless \emph{all} of the three terms in the minimum are large: to have high power,
we would need $B_{\textnormal{train}}\gg n$
to enable sufficiently many calls to $\cA$, 
and moreover, to search over $\cY$ space we need either $N_\ell\gg n$ or $|\cY|$ to be relatively small,
and similarly to search over $\cX$ space we need either $N_\ell+N_u\gg n$ or $|\cX|$ to be relatively small.
Note the second and third terms are statistical limits, due to limited data (i.e., $N_\ell$ and $N_u$), but 
these statistical limits can be overcomed if $|\cX|$ and $|\cY|$ are small enough to permit exhaustive search (under the computation constraint $B_{\textnormal{train}}$). In practice, we might expect the algorithm to become more stable as $n$ increases, and we will be interested in testing whether the algorithm is $(\epsilon_n, \delta_n)$-stable with $\epsilon_n, \delta_n \to 0$ as $n \to \infty$. However, Theorem \ref{thm:limits_evaluate-random} suggests that as long as one of the three terms in \eqref{ineq:power-bound} is small, e.g., when $B_{\textnormal{train}} = O(n)$ or $N_\ell = O(n)$, then it is even impossible to tell whether $\epsilon_n, \delta_n$ are smaller than a constant, let alone tell they go to zero.

\paragraph{A nearly matching bound.} We will now show that the upper bound on the power given in Theorem \ref{thm:limits_evaluate-random} can be approximately achieved, in certain settings, with the simple test described in Section \ref{sec:def-black-box-test} based on sample-splitting. Recall that the binomial test in Section \ref{sec:def-black-box-test} is based on the empirical count
\begin{equation*}
	B = \sum_{k=1}^{ \lfloor \kappa \rfloor } \indi( \Delta^{(k)} > \epsilon ),
\end{equation*} where $\kappa$ and $\Delta^{(k)}$ are given in \eqref{def:kappa} and \eqref{eq:delta-k}, respectively. Then we can define our test by comparing $B$ again a Binomial$( \lfloor \kappa \rfloor, \delta)$ distribution:
\begin{equation}\label{eqn:binomial-test}
	\begin{split}
		\hT_{\textnormal{Binom}}(\cA,\cD_{\ell}, \cD_u) = \begin{cases} 
1, & B < k_*,\\
1, & B = k_* \textnormal{ and }\zeta\leq a_*,\\
0, & \textnormal{ otherwise},
\end{cases}
	\end{split}
\end{equation} for a random value $\zeta\sim\textnormal{Unif}[0,1]$,
where the nonnegative integer $k_*$ and value $a_*\in[0,1)$ are chosen as the unique values satisfying
\[
	\bbP( \textnormal{Binomial}( \textstyle{\lfloor \kappa \rfloor}, \delta  )  < k^* ) + a^*\cdot  \bbP( \textnormal{Binomial}( \textstyle{\lfloor \kappa \rfloor}, \delta  )  = k^* ) = \alpha.
\] Then a similar result as the Theorem 1 in \cite{kim2021black} can be established.
\begin{Proposition} \label{prop:test-power} The test $\hT_{\textnormal{Binom}}$ is a black-box test under the computational constraint,
and satisfies the assumption-free validity condition~\eqref{eqn:validity}. Moreover, if $\delta < 1 - \alpha^{1/\lfloor \kappa \rfloor }$, then the power of the test is given by
\begin{equation*}
	\bbP( \hT_{\textnormal{Binom}}(\cA,\cD_{\ell}, \cD_u)  = 1) = \alpha \cdot \left( \frac{1 - \delta^*_{\epsilon} }{1 - \delta} \right)^{ \Big\lfloor \frac{B_{\textnormal{train} }}{n} \wedge \frac{N_{\ell}}{n} \wedge \frac{N_{\ell} + N_u}{n}  \Big\rfloor }.
\end{equation*}
\end{Proposition} 

Comparing this power calculation to the result of Theorem~\ref{thm:limits_evaluate-random},
we can see that in a setting where $B_{\textnormal{train} }$ is relatively small comparing with $|\cX|$ and $|\cY|$, the power of the simple test $\hT_{\textnormal{Binom}}$
is nearly as high as the optimal power that can be achieved by any black-box test.

\paragraph{An unlimited budget?}
While our result is stated for a black-box test run under a prespecified computational constraint, $B_{\textnormal{train}}$,
it is also natural to ask what power might be obtained if we do not specify a budget. In fact, our result applies
to this setting as well. Suppose that we require our black-box test $\hT$ to terminate at a finite time (i.e., $\hr$ must be finite),
but we do not specify a budget---that is, $\hr$ can be arbitrarily large, and $\sum_{r=1}^{\hr}|\cD^{(r)}_\ell|$ can be arbitrarily large.
In this setting, we can simply apply the results of Theorem~\ref{thm:limits_evaluate-random} while taking $B_{\textnormal{train}}\to\infty$.
In particular, 
\[\lim_{B_{\textnormal{train}}\to\infty}\left(1 - \frac{B_{\textnormal{train}}}{|\cY|}\wedge1\right) = \begin{cases} 1, & |\cY|=\infty, \\ 0, & |\cY|<\infty,\end{cases}\]
and the same is true for $|\cX|$. Therefore, even with an unlimited budget, if we require $\hT$ to return an answer in finite time then
Theorem~\ref{thm:limits_evaluate-random} implies the following bounds on power:
\[\bbP \{ \hT = 1 \} \leq \alpha  \left(\frac{1-\delta^*_\epsilon}{1-\delta}\right)^{\frac{N_\ell}{n}}, \textnormal{ if $|\cY|=\infty$; \quad
}\bbP \{ \hT = 1 \} \leq \alpha  \left(\frac{1-\delta^*_\epsilon}{1-\widetilde{\delta} }\right)^{\frac{N_\ell+N_u}{n}}, \textnormal{ if $|\cX|=\infty$}.\]

\paragraph{Comparison with existing work.} Our result is a strict generalization of 
\citet[Theorem 3]{kim2021black}	, which proves that power is bounded as
$
		\bbP \{ \hT = 1 \}  \leq \alpha \left( \frac{1 - \delta^*_{\epsilon} }{1 - \delta} \right)^{ \frac{N_{\ell}}{n} }
$ whenever $|\cY|=\infty$.\footnote{More precisely, their work assumes that the cardinality of the space $\cY$ is uncountably infinite,
but this is because they allow the test $\hT$ to run for countably infinitely many rounds. Since here we require $\hT$ to terminate after
finitely many rounds, under our definition of a black-box test, their result would hold also for a countably infinite $\cY$.}
Of course, this is a special case of our main result, obtained by taking $|\cY|=\infty$ and $B_{\textnormal{train}}\to\infty$.
The existing result in their work does not provide a bound on power when $|\cY|<\infty$ (for example,
their result does not apply in the common binary classification setting $\cX = \bbR^d$ and $\cY=\{0,1\}$), and does not
account for the role of the computational budget. 
The proof of our main result leverages some of the constructions developed in \cite{kim2021black}'s proof,
but because of the additional considerations on the size of data space as well as the computational constraint, the proof of Theorem \ref{thm:limits_evaluate-random} requires novel technical tools, as we will see in Section~\ref{sec:main-proof}.

\section{Discussion} \label{sec:discussion}
In this paper, we provide a unified power upper bound for any universally valid black-box test for testing algorithmic stability under computational constraints. We reveal the computational budget can be the dominant constraint when the budget $B_{\textnormal{train}}$ is small. At the same time, we illustrate that when the data space is finite, if the sample size is limited then certifying stability is possible only if the data spaces $\cX$ and $\cY$ are small enough to essentially permit an exhaustive search. 

As we have mentioned in Section \ref{sec:def-black-box-test} for the black-box test presented in Definition \ref{def:black-box-test}, while the training algorithm $\cA$ can only be studied empirically via calls to $\cA$ (that is, we cannot examine its theoretical properties, or, say, read the code that defines the implementation of $\cA$), in contrast the fitted model $\hf^{(r)}$ returned by $\cA$ at the $r$th round \emph{is} transparent to us---we know its analytic expression in constructing the final test. This is a less restrictive notion of black-box testing, giving more information to the analyst. This is called a black-box test with {\it transparent models} by \cite{kim2021black}. An even stronger constraint is the notion of a black-box test with {\it black-box models}, where the fitted models cannot be studied analytically and instead can only be studied via evaluating the models on finitely many test points. In Appendix \ref{app:black-box-test-black-box-model}, we also establish the power upper bound for black-box test with black-box models in testing algorithmic stability.

There are also a couple of interesting directions worth exploring in the future. First, as we mentioned in Section \ref{sec:main-results} that the power achieved by the simple binomial test is only tight when $B_{\textnormal{train}}$ is relatively small to $|\cX|$ and $|\cY|$, it is interesting to see whether we could develop procedures which can close the gap between the current upper and lower bounds on power in general scenarios. Second, our hardness result on testing algorithmic stability suggests that we should look for relaxation for solving this testing problem. One possible way would be to relax the notion of stability we consider. For example, in Section 5.2 of \cite{kim2021black}, the authors discussed it might be possible to test a weaker notion of data-conditional stability. The second possible way would be to relax the test we consider. It would be interesting to consider a partially black-box setting, where some part of the algorithm's mechanism is known and some part is unknown---for instance, this arises when we use a complex black-box algorithm for building features, but then run a simple regression to use these features for prediction. Finally, in this work, we focus on testing $(\epsilon, \delta)$-stability defined in Definition \ref{def:stability}; another popular stability notion in learning theory is the hypothesis stability proposed in \cite{bousquet2002stability}. The distinction between these two notions is that our notion of stability is defined based on the perturbation of prediction, whille hypothesis stability is defined based on the perturbation of loss. An interesting open question is whether there is a unified way to treat different notions of stability under the black-box testing framework.

\section{Proof of Theorem \ref{thm:limits_evaluate-random}} \label{sec:main-proof}
The theorem can be separated into three claims and we will prove each of these claims separately. Respectively, these three results reflect the bounds on power that are due to limited $Y$ data and $X$ data (statistical limits), and to a limited number of calls to $\cA$ (a computational limit). We will begin with the power upper bound due to limited $Y$ data.

\subsection{Statistical limit: bounding power due to limited $Y$ data} \label{sec:Y-space-pf}
The high-level idea of the proof is the following: Step 1, we show that for any $y \in \cY$, there is a event such that on that event $\{\cE_y \}$, the power is lower, i.e., $\bbP(\hT =1; \cE_y ) \leq  \alpha\left(\frac{1-\delta^*_\epsilon}{1-\delta}\right)^{\frac{N_\ell}{n}}$; Step 2, we show if the result in Step 1 is true, then we could integrate these events under the computational constraint and the ultimate power will also be low as well. 

The key lemma we need in Step 1 is the following. For any $y\in\cY$, define $\cE_y$ as the event that $y\not\in\cD^{(1)}_\ell \cup\dots\cup\cD^{(\hr)}_\ell$, when running the test $\hT$. Then
\begin{Lemma}\label{lem:Y}
In the setting of Theorem~\ref{thm:limits_evaluate-random},
for any $y\in\cY$,
$\bbP\{\hT = 1; \cE_y\} \leq \alpha\left(\frac{1-\delta^*_\epsilon}{1-\delta}\right)^{\frac{N_\ell}{n}}$.
\end{Lemma} The proof of Lemma \ref{lem:Y} is postponed to appendix, but the proof idea of Lemma \ref{lem:Y} is given as follows: we will construct a corrupted distribution $P'$, and a modified algorithm $\mathcal{A}'$ so that $\mathcal{A}'$ is not stable when the input data consists of $y$. Then  we could show under proper conditions on $(\cA',P')$, $(\mathcal{A}', P', n)$ is not stable, and it is difficult to distinguish between $\mathcal{A}', P'$ and $\mathcal{A},P$ given the event $\cE_y$ happens. Finally, the low power of test on $(\mathcal{A}, P, n)$ comes from the validity of the test when it examines the unstable tuple $(\mathcal{A}',P',n)$.

Then we can integrate the result in Lemma \ref{lem:Y} on different events $\cE_y$ and get the bound. Without loss of generality, we can assume $|\cY|> B_{\textnormal{train}}$,
since otherwise the bound is trivial. 
Fix any integer $M$ with $B_{\textnormal{train}}<M\leq |\cY|$  (note that $|\cY|$ may be finite or infinite),
and let $y_1,\dots,y_M\in\cY$ be distinct.
We then have \begin{multline*}
\sum_{i=1}^M \bbP\{\hT = 1; \cE_{y_i}\}
= \bbE\left[\sum_{i=1}^M \indi\{\hT = 1; \cE_{y_i}\}\right]
= \bbE\left[\indi\{\hT=1\} \cdot \sum_{i=1}^M \indi\{\cE_{y_i}\}\right]\\
\geq \bbE\left[\indi\{\hT=1\} \cdot (M - B_{ \textnormal{train} })\right] = \bbP\{\hT=1\} \cdot  (M - B_{ \textnormal{train} }),
\end{multline*}
where the inequality holds since, due to the computational constraint,
$y\in\cD^{(1)}_\ell \cup\dots\cup\cD^{(\hr)}_\ell$ can hold for at most $B_{\textnormal{train}}$ many values $y\in\cY$---and therefore,
at least $M-B_{\textnormal{train}}$ many of the events $\cE_{y_1},\dots,\cE_{y_M}$ must hold.
Therefore,
\[\bbP\{\hT=1\}  \leq \frac{\sum_{i=1}^M \bbP\{\hT = 1; \cE_{y_i}\}}{M - B_{ \textnormal{train} }} 
\leq\frac{M\cdot  \alpha\left(\frac{1-\delta^*_\epsilon}{1-\delta}\right)^{N_\ell/n}}{M - B_{ \textnormal{train} }} ,\]
where the last step applies Lemma~\ref{lem:Y}.
Since this holds for any $M$ with $B_{\textnormal{train}}<M\leq |\cY|$, we have
\[\bbP\{\hT=1\}  \leq \alpha\left(\frac{1-\delta^*_\epsilon}{1-\delta}\right)^{N_\ell/n}\cdot\inf_{M \, : \, B_{\textnormal{train}}<M\leq |\cY|}
\frac{M}{M- B_{ \textnormal{train} }} .\]
If $|\cY|<\infty$, then $\inf_{M \, : \, B_{\textnormal{train}}<M\leq |\cY|}
\frac{M}{M- B_{ \textnormal{train} }} = \frac{|\cY|}{|\cY| -  B_{ \textnormal{train}}} $.
If instead $|\cY| = \infty$, then $\inf_{M \, : \, B_{\textnormal{train}}<M\leq |\cY|}
\frac{M}{M- B_{ \textnormal{train} }} =1 $. For both cases, then, we have obtained the desired result.

\subsection{Statistical limit: bounding power due to limited $X$ data} \label{sec:X-space-bound-pf}
Analogously to our proof for limited $Y$ data,  for any $x\in\cX$ we define $\cE_x$ as the event that $x\not\in\cD^{(1)}_\ell \cup\dots\cup\cD^{(\hr)}_\ell$, when running the test $\hT$.
Again, the key step is to show the following result on finding many events so such the power is lower on those events:
\begin{Lemma}\label{lem:X}
In the setting of Theorem~\ref{thm:limits_evaluate-random}, if $\delta (1 + \frac{1}{en}) < 1$, then
for any $x\in\cX$, $\bbP\{\hT = 1; \cE_x\} \leq \alpha\left(\frac{1-\delta^*_\epsilon}{1-\delta(1+\frac{1}{en})}\right)^{\frac{N_\ell+N_u}{n}}.$
\end{Lemma} The proof of Lemma \ref{lem:X} is similar to Lemma \ref{lem:Y} and relies on a construction of $\cA'$ such that it becomes unstable when the training data consists of $x$. The details are provided in the appendix. 
With this lemma in place, the proof for the power upper bound due to limited data in $\cX$ space
is exactly the same as for the previous step (i.e., with $Y$ in place of $X$), so we omit the details.

\subsection{Computational limit: bounding power due to limited calls to $\cA$} \label{sec:Btrain-limit-proof}
In this section, we prove the power upper bound due to a limited number of calls to $\cA$. Similar to the proof in Section \ref{sec:Y-space-pf}, we again aim to find many events so that the power on those events is low. But the distinct from the previous proof, here we look for rare event based on the random seed of the algorithm. 

Recall that when we run the test $\hT$, the calls $r=1,\dots,\hr$ to the algorithm $\cA$
are carried out with data set $\cD^{(r)}_\ell$ and with random seed $\xi^{(r)}$ (which may be generated in any fashion).
Fixing any subset $R\subset[0,1]$, define $\cE_R$ as the event that $\xi^{(r)}\not\in R$ for all $r$ such that $|\cD^{(r)}_\ell|=n$---that is, for all calls $r$ to the algorithm
during which the training set size is equal to $n$.
Define also the function $f(\xi) = \bbP\left( | \cA(\cD_n;\xi)(X_{n+1}) - \cA(\cD_{n-1};\xi)(X_{n+1}) | > \epsilon \mid \xi \right)$, 
the conditional probability of an unstable prediction, as a function of the random seed $\xi$ (here the probability is taken with respect
to training and test data sampled i.i.d.\ from $P$). By definition, we have $\bbE[f(\xi)] = \delta^*_\epsilon$, where the expected
value is taken with respect to $\xi\sim\textnormal{Unif}[0,1]$. Then we can show a similar result as in Lemmas \ref{lem:Y} and \ref{lem:X}, with some mild assumptions on $R$, the power of $\hT$ given $\cE_R$ happens is low.
\begin{Lemma}\label{lem:Btrain}
In the setting of Theorem~\ref{thm:limits_evaluate-random},
for any measurable set $R\subset[0,1]$, if $\textnormal{ Leb}(R) + (1-\textnormal{Leb}(R)) \cdot \bbE[f(\xi)\mid \xi\not\in R] >\delta$,
then
$\bbP\{\hT = 1; \cE_R\} \leq \alpha$, where $\textnormal{Leb}(R)$ denote the Lebesgue measure of $R$.
\end{Lemma}
The proof of Lemma \ref{lem:Btrain} is provided in the appendix. With this lemma in place, we are ready to bound the power.
Fixing any integer $M\geq 1$, standard measure-theoretic arguments (e.g., Lemma 2.5 of Chapter 2 in \cite{bennett1988interpolation}) allow us to construct a partition $[0,1] = S_1\cup\dots\cup S_M$,
where $S_i$ is measurable and has Lebesgue measure $1/M$, and $\bbE[f(\xi)\mid \xi\in S_i] =\bbE[f(\xi)]= \delta^*_\epsilon$, for each $i\in[M]$.

Next, fix any integer $M_0\leq M$ with $M_0 >  M \cdot \frac{\delta - \delta_\epsilon^*}{1- \delta_\epsilon^*} $. For any subset $I\subset [M]$ with $|I| = M_0$, define a region $R_I = \cup_{i\in I} S_i$,
and observe that  $\textnormal{Leb}(R_I) = M_0/M > \frac{\delta - \delta_\epsilon^*}{1- \delta_\epsilon^*} $ by construction, and $\bbE[f(\xi)\mid \xi\not\in R_I] = \delta^*_\epsilon$ by construction of the partition $S_1\cup\dots\cup S_M$. 
Then by Lemma~\ref{lem:Btrain},
$\bbP\{\hT = 1; \cE_{R_I}\} \leq \alpha$.
Now we sum over all subsets $I$:
\[\sum_{I\subset[M], |I|=M_0} \bbP\{\hT = 1; \cE_{R_I}\} 
= \bbE\left[\indi\{\hT=1\}\cdot \!\!\!\!\sum_{I\subset[M], |I|=M_0}\!\!\!\! \indi\{\cE_{R_I}\} \right]
\geq
 \bbP\{\hT=1\}\cdot {M - \lfloor\frac{B_{\textnormal{train}}}{n}\rfloor  \choose M_0},
\]
where the inequality holds as follows: defining $\widehat{I}\subset[M]$ as the 
set of all indices $i\in[M]$ such that $\{ |\cD^{(r)}_\ell|=n$ and $\xi^{(r)}\in S_i\}$ holds for some $r\in\{1,\dots,\hr\}$,
we have $ \sum_{I\subset[M], |I|=M_0} \indi\{\cE_{R_I}\} = {M-|\widehat{I}|\choose M_0}$. And, due to the computational constraint, 
there are at most $\lfloor \frac{B_{\textnormal{train}}}{n}\rfloor$ many rounds $r$ for which $|\cD^{(r)}_\ell|=n$ and so we must have $|\widehat{I}|\leq  \lfloor \frac{B_{\textnormal{train}}}{n}\rfloor $.
Combining all our calculations, then,
\[\bbP\{\hT=1\} \leq  \frac{ \sum_{I\subset[M], |I|=M_0} \bbP\{\hT = 1; \cE_{R_I}\} }{{M - \lfloor \frac{B_{\textnormal{train}}}{n}\rfloor  \choose M_0}}
\leq\frac{{M\choose M_0}\cdot \alpha}{{M - \lfloor \frac{B_{\textnormal{train}}}{n}\rfloor  \choose M_0}}.\]
Finally, taking $M\to \infty$ (and setting $M_0 = 1 + \lfloor  M \cdot \frac{\delta - \delta_\epsilon^*}{1- \delta_\epsilon^*} \rfloor \approx  M \cdot \frac{\delta - \delta_\epsilon^*}{1- \delta_\epsilon^*} $), 
we have $\lim_{M\to\infty} \frac{{M\choose M_0}}{{M - \lfloor \frac{B_{\textnormal{train}}}{n}\rfloor  \choose M_0}} = \left(\frac{1-\delta^*_\epsilon}{1-\delta}\right)^{\lfloor \frac{B_{\textnormal{train}}}{n}\rfloor}$,
which proves that
$\bbP\{\hT=1\} \leq \alpha \left(\frac{1-\delta^*_\epsilon}{1-\delta}\right)^{ \frac{B_{\textnormal{train}}}{n}}$.

\subsection*{Acknowledgements}
R.F.B.\ was supported by the Office of Naval Research via grant N00014-20-1-2337
and by the National Science Foundation via grant DMS-2023109. 

\bibliographystyle{apalike}
\bibliography{reference.bib}

\appendix
\section{Extension of black-box test for algorithmic stability with black-box models} \label{app:black-box-test-black-box-model}
In this section, we turn to a different setting, where both the algorithm $\cA$ and its outputs (the fitted models $\hf$)
are treated as black-box. This setting can also be useful in some scenarios. Consider an algorithm $\cA$ returning a fitted model $\hf$
that is given by neural network with many layers, then our ability to study the behavior of this complex $\hf$ analytically is likely to be
limited---for example, it is known that computing the Lipschitz constant of a deep neural network is hard \citep{virmaux2018lipschitz}.
In many settings we may only be able to study properties of the fitted model $\hf$  empirically,
 i.e., by sampling a large number of values of $X$ and evaluating $\hf$ on these values.
(In contrast, if $\cA$ is a complex variable selection procedure but returns only simple models $\hf$, such as linear models,
then we would be able to study $\hf$ analytically---the transparent-model setting that we have considered up until now, in this paper.)

To allow for the setting of black-box models, we now modify our definition of a black-box test, as follows.
To help compare to Definition~\ref{def:black-box-test}, we use boldface text to highlight the parts of the definition that have been modified.
\begin{Definition}[Black-box test \textbf{with black-box models}, under computational constraints] \label{def:black-box-test-black-box-model} 
Consider any function $\hT$ that takes as input an algorithm $\cA$, a labeled data set $\cD_{\ell}\in \mathfrak{D}_{\ell}$, and an unlabeled data set $\cD_u \in \mathfrak{D}_u$,
and returns a (possibly randomized) output $\hT(\cA,\cD_{\ell}, \cD_u)\in\{0,1\}$. 
We say that $\hT$ is a black-box test, with computational budgets $B_{\textnormal{train}} >0$ for  model training \textbf{and $\boldsymbol{B_{\textnormal{eval}} > 0}$ for model evaluation}, if it can be defined as follows:  
 for some functions $g^{(1)},g^{(2)},\dots$ and $g$, and for random seeds, $\zeta^{(1)},\zeta^{(2)},\dots,\zeta\iidsim  \textnormal{Unif}[0,1]$,
 		 \begin{enumerate}[leftmargin=*]
	 	\item For each stage $r = 1,2,3,\ldots$, 
	 	\begin{enumerate}
	 	\item If $r = 1$, generate a new labeled data set $\cD_{\ell}^{(1)}$, \textbf{an unlabeled data set} $\boldsymbol{\cX^{(1)}}$, and a new random seed $\xi^{(1)}$,
		\[(\cD_{\ell}^{(1)}, \boldsymbol{\cX^{(1)}},\xi^{(1)} ) = g^{(1)} (\cD_{\ell}, \cD_u, \zeta^{(1)} ).\]
		 \begin{itemize}
		 	\item If some stopping criterion is reached or $|\cD^{(1)}_{\ell}| > B_{\textnormal{train}}$ \textbf{or $\boldsymbol{|\cX^{(1)}| > B_{\textnormal{eval}}}$}, exit the for loop; otherwise compute the fitted model $\hf^{(1)} = \cA( \cD_{\ell}^{(1)}; \xi^{(1)} )$, \textbf{and evaluate the fitted model by computing $\boldsymbol{\hf^{(1)}(\cX^{(1)})}$.}\footnote{Here $\hf^{(1)}(\cX^{(1)})$ is understood in the pointwise sense, i.e., we compute $\hf^{(1)}(x)$ for each $x\in\cX^{(1)}$.}
		 \end{itemize}
	 	\item If $r \geq 2$, generate a labeled data set $\cD_{\ell}^{(r)}$, \textbf{an unlabeled data set} $\boldsymbol{\cX^{(r)}}$, and a new random seed $\xi^{(r)}$,
		\begin{multline*}(\cD_{\ell}^{(r)},\boldsymbol{\cX^{(r)}},\xi^{(r)} )\\ = g^{(r)} \big(\cD_{\ell}, \cD_u, (\cD_{\ell}^{(s)})_{1\leq s<r}, \boldsymbol{(\cX^{(s)})_{1\leq s<r}, (\hf^{(s)}(\cX^{(s)}))_{1\leq s<r}}, \\(\zeta^{(s)})_{1\leq s\leq r} , (\xi^{(s)})_{1\leq s<r}\big).\end{multline*}
		\begin{itemize}
			\item If some stopping criterion is reached or $\sum_{j = 1}^{r} |\cD_{\ell}^{(j)}| > B_{\textnormal{train}}$ \textbf{or $\boldsymbol{\sum_{j = 1}^{r} |\cX^{(j)}| > B_{\textnormal{eval}}}$}, exit the for loop; otherwise compute the fitted model $\hf^{(r)} = \cA( \cD_{\ell}^{(r)}; \xi^{(r)} )$, \textbf{and evaluate the fitted model by computing $\boldsymbol{\hf^{(r)}(\cX^{(r)})}$}.
		\end{itemize}
	 	\end{enumerate}
		\item Let $\hr$ denote the total number of rounds when stopping.
	 	\item Finally, return
		\[\hT(\cA,\cD_{\ell}, \cD_u) = g\big(\cD_{\ell}, \cD_u,(\cD_{\ell}^{(r)})_{1\leq r\leq \hr},\boldsymbol{(\cX^{(r)})_{1\leq r\leq \hr}},\boldsymbol{(\hf^{(r)}(\cX^{(r)}))_{1\leq r\leq \hr}},(\zeta^{(r)})_{1\leq r\leq \hr},(\xi^{(r)})_{1\leq r\leq \hr},\zeta\big).\]
		\end{enumerate}
\end{Definition}

A pictorial illustration of a black-box test with black-box models is given in Figure \ref{fig:black-box-alg-black-box-model}. The main result for the power upper bound of testing algorithmic stability via black-box test with black-box models is given as follows.
\begin{Theorem} \label{thm:limits_evaluate-black-box-model} Fix any $\epsilon \geq 0$, $\delta \in [0,1 )$. Let $\hT$ be any black-box test with black-box models, as in Definition~\ref{def:black-box-test-black-box-model}. Suppose $\hT$ has computational budgets $B_{\textnormal{train}}$ for model training and $B_{\textnormal{eval}}$  for model evaluation, and satisfies the assumption-free validity condition \eqref{eqn:validity} at level $\alpha$. Then for any $(\cA,P,n)$ that is $(\epsilon,\delta )$-stable (i.e. $\delta^*_{\epsilon} \leq \delta$ and the alternative holds), the power of $\hT$ is bounded as 
\begin{equation}\label{ineq:power-bound-black-box-model}
		\bbP \{ \hT = 1 \} \leq \alpha \cdot \min \Bigg\{  \ \underbrace{\left( \frac{1 - \delta^*_{\epsilon} }{1 - \delta} \right)^{\frac{B_{\textnormal{train} }}{n} } , \frac{\left( \frac{1 - \delta^*_{\epsilon} }{1 - \delta} \right)^{ \frac{N_{\ell}}{n} }}{1 - \frac{B_{\textnormal{train}}}{|\cY|} \wedge 1 } , \frac{\left( \frac{1 - \delta^*_{\epsilon} }{1 - \widetilde{\delta} }\right)^{ \frac{N_{\ell} + N_u}{n} }}{1 - \frac{B_{\textnormal{train}}}{|\cX|}\wedge 1 } }_{\textnormal{same bounds as for the transparent-model setting (Theorem~\ref{thm:limits_evaluate-random})}}, \underbrace{\frac{\left( \frac{1 - \delta^*_{\epsilon} }{1 -\delta }\right)^{ \frac{N_{\ell} + N_u}{n+1} }}{1 - \frac{B_{\textnormal{train}}+B_{\textnormal{eval}}}{|\cX|}\wedge 1 } }_{\substack{\textnormal{new constraint due to limited}\\\textnormal{budget for model evaluation}}} \ \Bigg\},
\end{equation} where $\widetilde{\delta} = (\delta + \frac{1}{n} ) \wedge 1$.
\end{Theorem}
Comparing with Theorem \ref{thm:limits_evaluate-random}, we now have an additional upper bound on the power due to limited budget for model evaluation.
\paragraph{Comparison to existing work.}
Similarly to the transparent-model setting, we can compare to the results of \cite{kim2021black} in the case of infinite data spaces $\cX$ and $\cY$.
When $|\cX|=|\cY|=\infty$, then even for arbitrarily large or unlimited $B_{\textnormal{train}}$ and $B_{\textnormal{eval}}$  (i.e., we may make arbitrarily many---but still finitely many---calls to train $\cA$ and to evaluate the fitted models), we have $\frac{B_{\textnormal{train}}}{|\cY|} , \frac{B_{\textnormal{train}}+B_{\textnormal{eval}}}{|\cX|} \to 0$ and so \eqref{ineq:power-bound-black-box-model} implies
\[
		\bbP \{ \hT = 1 \}  \leq \alpha \left( \frac{1 - \delta^*_{\epsilon} }{1 - \delta} \right)^{ \frac{N_{\ell}}{n} \wedge \frac{N_{\ell} + N_u}{n+1}  } .
\]This exactly recovers the result of \citet[Theorem 2]{kim2021black} and demonstrates the hardness of testing algorithmic stability when 
the available data is limited. 

\begin{figure}[t]
	\centering\medskip
	\fbox{
	\begin{tikzpicture}[
	roundnode/.style={circle, draw=gray!60, fill=gray!5, thick, minimum size=13mm},
	  >=stealth,
    node distance=3cm,
    auto
	]
		\node (left) [fill=blue!25,text width=3.3cm] at (-3.8,-1) {Input:\\ labeled data $\cD_{\ell}$,\\ unlabeled data $\cD_{u}$,\\ algorithm $\cA$,\\
	 seeds $\zeta,\zeta_1,\zeta_2,\dots$};
		\draw[-latex, line width=0.5mm] (-2,-1) -> (-1,-1);
		\draw[black, thick] (-1,-4.3) rectangle (4,2.1);
		\node (upper1) [text width=9cm] at (2,3.) {Iterate for rounds $r = 1,2,\ldots,\hr$, subject to computational constraints $\sum_{r=1}^{\hr}|\cD^{(r)}_\ell|\leq B_{\textnormal{train}}$  and $\sum_{j = 1}^{\hr} |\cX^{(r)}| \leq B_{\textnormal{eval}}$};
			\node (upper) [fill=blue!25,text width=4.5cm] at (1.5,0.9) {Generate labeled data set $\cD^{(r)}_{\ell}$,  random seed $\xi^{(r)}$ and an unlabeled data set, $\cX^{(r)}$};
		\node (down) [fill=blue!25,text width=4.5cm] at (1.5,-2.9) {Call the algorithm to train a model: $\hf^{(r)} = \cA(\cD^{(r)}_{\ell};\xi^{(r)})$, then evaluate $\hf^{(r)}(\cX^{(r)})$};
		\draw[ultra thick,->] (0,-0.1) .. controls (-0.5,-0.95) .. (0.,-1.8);
		\draw[ultra thick,->] (2.9,-1.8) .. controls (3.4,-0.95) .. (2.9,-0.1);
		\draw[-latex, line width=0.5mm] (4.,-1) -> (4.8,-1);
		\node (right) [fill=blue!25,text width=3.7cm] at (6.8,-1) {Output:\\ $\hT(\cA,\cD_{\ell}, \cD_u ) \in \{0,1\}$};
	\end{tikzpicture}
}
	
	\caption{Black-box test with black-box models under computational constraints. Here in round $r$, the generating process of the new labeled data set $\cD_{\ell}^{(r)}$ for training, the unlabeled data set $\cX^{(r)}$ for evaluation and the new random seed $\xi^{(r)}$ can depend on the input data and on all the past information from previous rounds---e.g., we may create a new labeled data set and an unlabeled data set by resampling from past data. The input random seeds $\zeta,\zeta_1,\zeta_2,\dots\iidsim\textnormal{Unif}[0,1]$ may be used to provide randomization if desired. The procedure stops at a data-dependent stopping time $\hr$ (i.e., we may use any stopping criterion), subject to the computational constraints that (1) the total number of training points used in calls to $\cA$ cannot exceed the budget $B_{\textnormal{train}}$ and (2) the total number of data points used to evaluate the fitted models cannot exceed the budget $B_{\textnormal{eval}}$.
	}
	\label{fig:black-box-alg-black-box-model}
\end{figure}
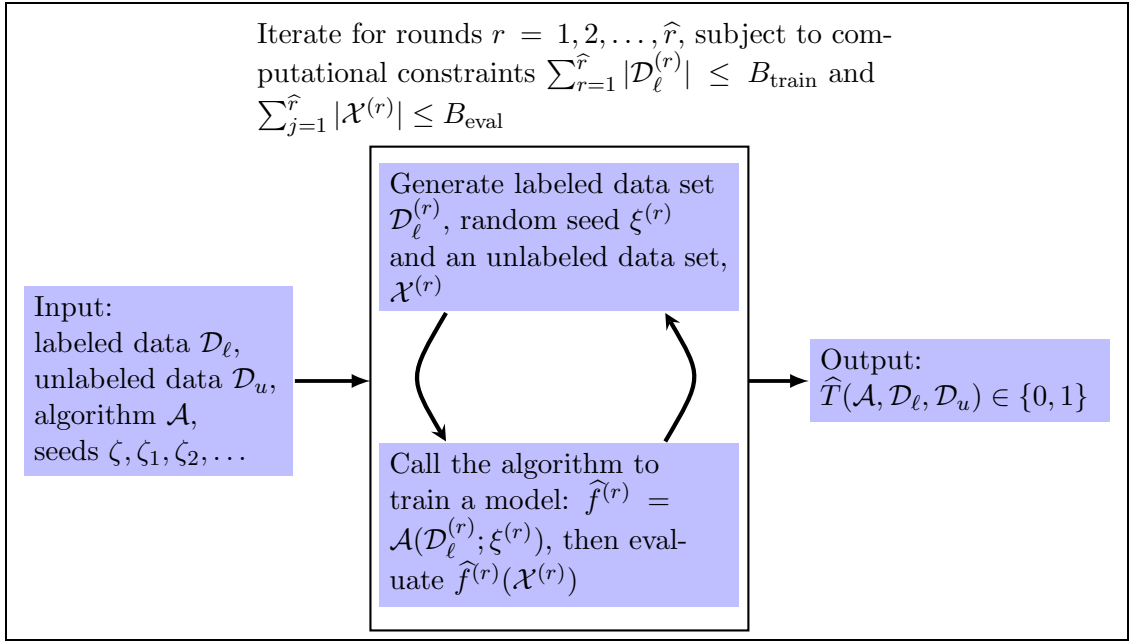

\section{Proof of Lemma \ref{lem:Y} }
At a high level, the main idea of the proof of this part
is to construct a corrupted distribution $P'$, and a modified algorithm $\cA'$, so that $(\cA',P',n)$ is \emph{not} stable,
but it is difficult to distinguish between $\cA',P'$ and $\cA,P$ given limited data. This proof extends the ideas constructed in
the proof of \cite[Theorem 2]{kim2021black} to this more complex setting where $|\cY|$ is no longer assumed to be infinite.

Fix any $c\in[0,1]$.
Define a new distribution
\[P' = c \cdot (P_X\times\delta_y) + (1-c)\cdot P,\]
where $P_X$ is the marginal of $X$ under $(X,Y)\sim P$, while $\delta_y$ is the point mass at $y$---note that the marginal distribution of $X$ is the same under $P'$ as under $P$.
Now we define a new algorithm  $\cA'$ as 
\[\cA'(\cD;\xi) = \begin{cases}\cA_{1}(\cD; \xi), & \textnormal{ if }|\cD| = n \textnormal{ and }y\in \cD,\\
 \cA(\cD;\xi), & \textnormal{ otherwise} ,\end{cases}\]
 	where the informal notation $y\in\cD$ denotes that at least one of the response values in the data points contained in $\cD$ is equal to $y$,
	and $\cA_{1}(\cD; \xi)$ is the function given by $\cA_{1}(\cD; \xi)(x) = 1 + \epsilon + \max_{\sigma \in S_{|\cD|}} \cA( (\cD_{\sigma})_{-|\cD|}; \xi )(x)$. Here $S_{|\cD|}$ denotes the set of all permutations on a set of size $|\cD|$, and the notation $(\cD_{\sigma})_{-|\cD|}$ means that we first permute data set $\cD$, i.e., $\left( (x_{\sigma(1)}, y_{\sigma(1)}), \ldots, (x_{\sigma(|\cD|)}, y_{\sigma(|\cD|)} ) \right)$, and then remove the last data point. 
	In other words, we have constructed $\cA'$ so that if the training data set $\cD$ has $n$ data points and $y \in \cD$, then we will inevitably see instability, i.e., we will have $|\cA'(\cD_n; \xi)(x) - \cA'(\cD_{n-1}; \xi)(x) | > \epsilon$, for any $x$.

 Next we verify that $(\cA',P',n)$ is not stable. Let $\cD_n$ denote a training set of size $n$ drawn i.i.d.\ from $P'$.
 Let $\cE_{\cD_n\sim P}$ be the event
that, when sampling the training set $\cD_n$, all $n$ samples are drawn from the second component of the mixture, i.e., from $P$.
Then this event holds with probability $(1-c)^n$ by construction. Moreover, on the event that $\cE_{\cD_n\sim P}$ does not occur, we have
\[\bbP_{P'}\{ |\cA'(\cD_n;\xi)(X_{n+1}) - \cA'(\cD_{n-1};\xi)(X_{n+1})| > \epsilon \mid \cE_{\cD_n\sim P}^c\} =1\]
(since, if at least one training sample is drawn from the first component of the mixture, i.e., $P_X\times\delta_y$, then we have $y\in\cD_n$
and so instability must occur by construction of $\cA'$), while if instead $\cE_{\cD_n\sim P}$ does occur, we calculate
\begin{multline*}\bbP_{P'}\{ |\cA'(\cD_n;\xi)(X_{n+1}) - \cA'(\cD_{n-1};\xi)(X_{n+1})| > \epsilon \mid \cE_{\cD_n\sim P}\} \\
\geq \bbP_P\{ |\cA(\cD_n;\xi)(X_{n+1}) - \cA(\cD_{n-1};\xi)(X_{n+1})| > \epsilon\}  = \delta^*_\epsilon,\end{multline*}
where the inequality holds by construction of $\cA'$. Therefore,
\[\bbP_{P'}\{ |\cA'(\cD_n;\xi)(X_{n+1}) - \cA'(\cD_{n-1};\xi)(X_{n+1})| > \epsilon\}
\geq \left[ 1 - (1-c)^n\right] + (1-c)^n \delta^*_\epsilon,\]
and in particular, this quantity is $>\delta$ for any $ c > 1 - \left(\frac{1-\delta}{1-\delta^*_\epsilon}\right)^{1/n}$.
Finally, since $\cA'$ returns identical output to $\cA$ on each training set $\cD^{(r)}_\ell$ as long as $y\not\in\cD^{(r)}_\ell$,
therefore on the event $\cE_y$ (i.e., the event $y\not\in\cD^{(1)}_\ell\cup\dots\cup\cD^{(\hr)}_\ell$) we must  have $\hT(\cA',\cD_\ell,\cD_u)=1$ if and only if $\hT(\cA,\cD_\ell,\cD_u)=1$
(a more formal proof of this type of statement can be found in the proof of \citet[Theorem 2]{kim2021black}).
Combining these calculations, we have
\begin{multline*}
		\bbP_P\{\hT(\cA,\cD_\ell,\cD_u)=1; \cE_y\} = 
\bbP_P\{\hT(\cA',\cD_\ell,\cD_u)=1; \cE_y\} \leq \bbP_P\{\hT(\cA',\cD_\ell,\cD_u)=1\} \\
\leq (1-c)^{- N_\ell}\cdot \bbP_{P'}\{\hT(\cA',\cD_\ell,\cD_u)=1\} 
\leq \alpha (1-c)^{-N_\ell},
	\end{multline*}
where the last step holds by validity of $\hT$~\eqref{eqn:validity}, since $(\cA',P',n)$ is not $(\epsilon,\delta)$-stable.
We have established that this is true for any 
$ c > 1 - \left(\frac{1-\delta}{1-\delta^*_\epsilon}\right)^{1/n}$, and therefore,
\[\bbP_P\{\hT(\cA,\cD_\ell,\cD_u)=1; \cE_y\} 
\leq \inf_{c > 1 - \left(\frac{1-\delta}{1-\delta^*_\epsilon}\right)^{1/n}} 
 \alpha (1-c)^{-N_\ell}
 = \alpha \left(\frac{1-\delta^*_\epsilon}{1-\delta}\right)^{N_\ell/n}.\]

\section{Proof of Lemma \ref{lem:X} }
The proof of this lemma follows a similar construction as for the proof of Lemma~\ref{lem:Y} (for $Y$ in place of $X$),
but is more complex since the value of $X$ plays a role both in the training data and in the test point.

Fix any $c\in[0,1]$.
Define a new distribution
\[P' = c \cdot (\delta_x\times P_Y) + (1-c)\cdot P,\]
where $P_Y$ is the marginal of $Y$ under $(X,Y)\sim P$, while $\delta_x$ is the point mass at $x$. 
Now define a new algorithm  $\cA'$ as 
\begin{equation}\label{def:A-prime-x-train}
	\cA'(\cD;\xi) = \begin{cases}\cA_{1}(\cD; \xi), & \textnormal{ if }|\cD| = n \textnormal{ and }x\in \cD,\\
 \cA(\cD;\xi), & \textnormal{ otherwise} ,\end{cases}
\end{equation}
 	where $\cA_{1}$ is defined as in the proof of Lemma~\ref{lem:Y}, and where $x\in\cD$ is interpreted as before.
Now we see that, for an appropriate choice of $c$, $(\cA',P',n)$ is unstable. We have
\[ \bbP_{P'} \left\{ |\cA'(\cD_n; \xi)(X_{n+1}) - \cA'(\cD_{n-1}; \xi)(X_{n+1}) | > \epsilon  \right\} \geq \left[1 - (1-c)^n\right] + (1-c)^{n+1}\delta^*_\epsilon,\]
where,  similarly to the analogous calculation in the proof of Lemma~\ref{lem:Y}, the term in square brackets arises from the event that at least one of the $n$ training points is drawn
from the first component in the mixture distribution $P'$, i.e., from $\delta_x\times P_Y$,
while the remaining term arises from the event that all $n$ training points and the test point $X_{n+1}$
are drawn from the second component, $P$.
Next, we have
\[ 1 - (1-c)^n + (1-c)^{n+1}\delta^*_\epsilon
\geq 1 - (1-c)^n(1 - \delta^*_\epsilon) - \frac{\delta^*_\epsilon}{en}\geq 1 - (1-c)^n(1 - \delta^*_\epsilon) - \frac{\delta}{en},\]
where the first inequality holds since $ ((1-c)^n - (1-c)^{n+1})  \leq  \sup_{t\in[0,1]}\{t^n - t^{n+1}\} \leq \frac{1}{en}$,
and the second inequality holds since $\delta^*_\epsilon\leq \delta$.
Taking any $c>1- \left(\frac{1-\delta(1+ \frac{1}{en})}{1-\delta^*_\epsilon}\right)^{1/n}$, we have $$\bbP_{P'} \left\{ |\cA'(\cD_n; \xi)(X_{n+1}) - \cA'(\cD_{n-1}; \xi)(X_{n+1}) | > \epsilon  \right\} >\delta.$$

Finally, since $\cA'$ returns identical output to $\cA$ on each training set $\cD^{(r)}_\ell$ as long as $x\not\in\cD^{(r)}_\ell$,
on the event $\cE_x$ we must  have $\hT(\cA',\cD_\ell,\cD_u)=1$ if and only if $\hT(\cA,\cD_\ell,\cD_u)=1$.
Then\begin{multline*}\bbP_P\{\hT(\cA,\cD_\ell,\cD_u)=1; \cE_x\} = 
\bbP_P\{\hT(\cA',\cD_\ell,\cD_u)=1; \cE_x\} \leq \bbP_P\{\hT(\cA',\cD_\ell,\cD_u)=1\} \\
\leq (1-c)^{- (N_\ell+N_u)}\cdot \bbP_{P'}\{\hT(\cA',\cD_\ell,\cD_u)=1\} 
\leq \alpha (1-c)^{-(N_\ell+N_u)},\end{multline*}
where the last step holds by validity of $\hT$~\eqref{eqn:validity}, since $(\cA',P',n)$ is not $(\epsilon,\delta)$-stable.
We have established that this is true for any $c\in[0,1]$ with $ c> 1-\left(\frac{1-\delta (1+ \frac{1}{en})}{1-\delta^*_\epsilon}\right)^{1/n}$, this proves the desired bound.

\section{Proof of Lemma \ref{lem:Btrain} }
We define an algorithm $\cA'$, that is a slight modification of $\cA$. The idea will be that $\cA'$ is unstable,
but cannot be distinguished from $\cA$ on the event $\cE_R$. We define\footnote{In this construction,
we are using the fact that algorithms produce functions that map from $\cX$ to $\bbR$, as in~\eqref{eqn:define_alg}---i.e., fitted models return real-valued
predictions. In some settings, for instance if $\cY=\{0,1\}$ or $\cY=[0,1]$, it may be more natural to constrain
fitted models to return predictions in a bounded subset of $\bbR$, e.g., in $[0,1]$. In that type of setting, the same
proof techniques for all our results still apply with a slight modification
of the construction of the modified algorithm $\cA'$. The resulting bound on power is slightly weaker, but the difference
vanishes as $\epsilon\to0$. 
}
 \begin{equation*}
	\cA'(\cD;\xi) = \begin{cases}\cA_{1}(\cD; \xi), & \textnormal{ if }|\cD| = n \textnormal{ and } \xi \in R,\\
 \cA(\cD;\xi), & \textnormal{ otherwise},\end{cases}
\end{equation*}  where $\cA_1$ is defined as in the proof of Lemma~\ref{lem:Y}.
Next, we verify that $(\cA',P,n)$ is not stable: we have
\begin{align*}
		& \bbP_{P} \left\{ |\cA'(\cD_n; \xi)(X_{n+1}) - \cA'(\cD_{n-1}; \xi)(X_{n+1}) | > \epsilon  \right\} \\
		&=   \bbP_{P} \left\{ |\cA'(\cD_n; \xi)(X_{n+1}) - \cA'(\cD_{n-1}; \xi)(X_{n+1}) | > \epsilon , \xi \in R \right\} \\
		& \hspace{1in} {}+ \bbP_{P} \left\{ |\cA'(\cD_n; \xi)(X_{n+1}) - \cA'(\cD_{n-1}; \xi)(X_{n+1}) | > \epsilon , \xi \notin R   \right\}\\
		&=   \bbP_{P} \left\{\xi \in R \right\} + \bbP_{P} \left\{ |\cA(\cD_n; \xi)(X_{n+1}) - \cA(\cD_{n-1}; \xi)(X_{n+1}) | > \epsilon , \xi \notin R   \right\},
	\end{align*}
	since by construction of $\cA'$, if $\xi\in R$ then we must have instability;
	while if $\xi\not\in R$ then $\cA'$ returns the same fitted model as $\cA$.
	Moreover, since $\xi\sim\textnormal{Unif}[0,1]$, we have $\bbP_{P} \left\{\xi \in R \right\} = \textnormal{Leb}(R)$.
	And, by definition of the function $f$, $\bbP_{P} \left\{ |\cA(\cD_n; \xi)(X_{n+1}) - \cA(\cD_{n-1}; \xi)(X_{n+1}) | > \epsilon\mid \xi\not\in R\right\} = \bbE_P[f(\xi)\mid \xi\not\in R]$. Combining everything, then,
	\[ \bbP_{P} \left\{ |\cA'(\cD_n; \xi)(X_{n+1}) - \cA'(\cD_{n-1}; \xi)(X_{n+1}) | > \epsilon  \right\}  =  \textnormal{Leb}(R) + (1- \textnormal{Leb}(R))\cdot  \bbE_P[f(\xi)\mid \xi\not\in R].\] Since this quantity is assumed to be $>\delta$, by the assumption in the lemma,
	this verifies  that $(\cA',P,n)$ is not $(\epsilon,\delta)$-stable. Thus by validity of $\hT$~\eqref{eqn:validity} we have
	$\bbP_P(\hT(\cA',\cD_\ell,\cD_u)=1)\leq \alpha$.
Next, observe that when we run the test $\hT$ on data sets $\cD_\ell$ and $\cD_u$ with either algorithm $\cA$ or $\cA'$, by construction of $\cA'$ the outcome of the test must be the same on the event $\cE_R$. So, we have
\[\bbP_P(\hT(\cA,\cD_\ell,\cD_u)=1; \cE_R)
= \bbP_P(\hT(\cA',\cD_\ell,\cD_u)=1; \cE_R) \leq \bbP_P(\hT(\cA',\cD_\ell,\cD_u)=1)\leq \alpha.\]

\section{Proof of Therorem \ref{thm:limits_evaluate-deter}}

First, in terms of the power upper bound due to the limited samples in $\cX$ and $\cY$, it is easy to check that the proof is the same as the corresponding part in Theorem \ref{thm:limits_evaluate-random}. Next, we show the power upper bound due to computational limit. We fix a partition of the data space into $M$ sets,
\[\cX\times\cY = \cup_{m=1}^M C_m,\]
and given the dataset $\cD_n = ((X_1,Y_1),\dots,(X_n,Y_n))$, we define its vector of counts as 
\[c(\cD_n) = (c_1(\cD_n),\dots,c_M(\cD_n))\textnormal{ where }
c_m(\cD_n) = \sum_{i=1}^n\One{(X_i,Y_i)\in C_m}.\]
That is, $c(\cD_n)$ is a vector of nonnegative integers summing to $n$, which reveals how the $n$ data points are partitioned over the sets $C_1,\dots,C_M$. This random vector takes values in the set
\[\mathbf{I}_{n,M} = \left\{\mathbf{i}=(i_1,\dots,i_M)\in\mathbb{Z}^M : i_1,\dots,i_M\geq 0, \sum_{m=1}^M i_m = n\right\}.\]

Now we are ready for our construction. For any $q\in\{0,1\}^{\mathbf{I}_{n,M}}$, define an algorithm $\cA'$ as 
\[\cA'(\cD,\xi) = \begin{cases}
    \cA_1(\cD,\xi), &\textnormal{ if $|\cD|=n$ and $q_{c(\cD)} = 1$},\\
    \cA(\cD,\xi), & \textnormal{ otherwise,}\end{cases}
\] where $\cA_{1}(\cD; \xi)(x) = 1 + \epsilon + \max_{\sigma \in S_{|\cD|}} \cA( (\cD_{\sigma})_{-|\cD|}; \xi )(x)$. First we check the stability of $\cA'$: for data sampled i.i.d.\ from $P$, we have
\begin{align*}
    &\PP{|[\cA'(\cD_n,\xi)](X_{n+1}) - [\cA'(\cD_{n-1},\xi)](X_{n+1})|\leq \epsilon}\\
    &=\PP{|[\cA'(\cD_n,\xi)](X_{n+1}) - [\cA'(\cD_{n-1},\xi)](X_{n+1})|\leq \epsilon, q_{c(\cD_n)} = 1} \\
    &\hspace{1in}+ \PP{|[\cA'(\cD_n,\xi)](X_{n+1}) - [\cA'(\cD_{n-1},\xi)](X_{n+1})|\leq \epsilon,q_{c(\cD_n)} = 0}\\
    &=\PP{|[\cA_1(\cD_n,\xi)](X_{n+1}) - [\cA_1(\cD_{n-1},\xi)](X_{n+1})|\leq \epsilon, q_{c(\cD_n)} = 1}\\
    &\hspace{1in} + \PP{|[\cA(\cD_n,\xi)](X_{n+1}) - [\cA(\cD_{n-1},\xi)](X_{n+1})|\leq \epsilon,q_{c(\cD_n)} = 0}\\
    &= \PP{|[\cA(\cD_n,\xi)](X_{n+1}) - [\cA(\cD_{n-1},\xi)](X_{n+1})|\leq \epsilon,q_{c(\cD_n)} = 0},
\end{align*}
by definition of $\cA'$ and of $\cA_1$. Now define
\[p_{\mathbf{i}} =  \PP{|[\cA(\cD_n,\xi)](X_{n+1}) - [\cA(\cD_{n-1},\xi)](X_{n+1})|\leq \epsilon, c(\cD_n) = \mathbf{i}},\]
for each $\mathbf{i}\in\mathbf{I}_{n,M}$.
Then 
\[ \One{q_{c(\cD_n)}=0} = \sum_{\mathbf{i}\in\mathbf{I}_{n,M}}(1-q_{\mathbf{i}})\One{c(\cD_n)=\mathbf{i}},\]
and so the above calculation can be rewritten as
\[
    \PP{|[\cA'(\cD_n,\xi)](X_{n+1}) - [\cA'(\cD_{n-1},\xi)](X_{n+1})|\leq \epsilon}\\
    = \sum_{\mathbf{i}\in\mathbf{I}_{n,M}} (1-q_{\mathbf{i}})\cdot p_{\mathbf{i}}.
\]
The distribution-free validity of $\hT$ implies that
\[\sum_{\mathbf{i}\in\mathbf{I}_{n,M}} (1-q_{\mathbf{i}})\cdot p_{\mathbf{i}}<1- \delta \ \Longrightarrow \ \PP{\hT(\cA',\cD_\ell,\cD_u) = 1}\leq \alpha.\]

Next, by definition of the black-box test $\hT$ and by construction of $\cA'$, the test must return the same answer for $\cA$ as for $\cA'$ on the event that the algorithm is \emph{not} called on any dataset $\cD$ for which $q_{c(\cD)}=1$. In other words,
\begin{multline*}\One{\hT(\cA,\cD_\ell,\cD_u) = 1, q_{c(\cD^{(1)}_\ell)} = \cdots = q_{c(\cD^{(\hr)}_\ell)} = 0} \\= \One{\hT(\cA',\cD_\ell,\cD_u) = 1, q_{c(\cD^{(1)}_\ell)} = \cdots = q_{c(\cD^{(\hr)}_\ell)} = 0} \leq \One{\hT(\cA',\cD_\ell,\cD_u) = 1},\end{multline*}
and therefore,
\begin{multline*}\PP{\hT(\cA,\cD_\ell,\cD_u) = 1, q_{c(\cD^{(1)}_\ell)} = \cdots = q_{c(\cD^{(\hr)}_\ell)} = 0} \leq \PP{\hT(\cA',\cD_\ell,\cD_u) = 1}\\
\leq \alpha + \One{\sum_{\mathbf{i}\in\mathbf{I}_{n,M}}  (1-q_{\mathbf{i}})\cdot p_{\mathbf{i}}\geq 1- \delta}.\end{multline*}
This result holds for any \emph{fixed} vector $q\in\{0,1\}^{\mathbf{I}_{n,M}}$. Now, from this point on, we will take $q$ to be random---let $q_{\mathbf{i}}\iidsim \textnormal{Bernoulli}(\rho)$, where $\rho$ will be chosen below. The result above then holds conditional on $q$---that is,
\begin{multline*}\PPst{\hT(\cA,\cD_\ell,\cD_u) = 1, q_{c(\cD^{(1)}_\ell)} = \cdots = q_{c(\cD^{(\hr)}_\ell)} = 0}{q}\\
\leq \alpha + \One{\sum_{\mathbf{i}\in\mathbf{I}_{n,M}} (1-q_{\mathbf{i}})\cdot p_{\mathbf{i}} \geq 1-\delta}.\end{multline*}
After marginalizing over $q$, then,
\begin{multline*}\PP{\hT(\cA,\cD_\ell,\cD_u) = 1, q_{c(\cD^{(1)}_\ell)} = \cdots = q_{c(\cD^{(\hr)}_\ell)} = 0}\\
\leq \alpha + \PP{\sum_{\mathbf{i}\in\mathbf{I}_{n,M}} (1-q_{\mathbf{i}})\cdot p_{\mathbf{i}} \geq 1-\delta}.\end{multline*}
Now we calculate
\[\PPst{q_{c(\cD^{(1)}_\ell)} = \cdots = q_{c(\cD^{(\hr)}_\ell)} = 0}{\hT(\cA,\cD_\ell,\cD_u),\cD^{(1)},\dots,\cD^{(\hr)}} \geq (1-\rho)^{\hr} \geq (1-\rho)^{\frac{B_{\textnormal{train}}}{n}},\]
with equality for the first step in the case that the values $c(\cD^{(1)}),\dots,c(\cD^{(\hr)})$ are all distinct, and with the last step holding due to the computational constraint $\hr\leq \frac{B_{\textnormal{train}}}{n}$. Therefore,
\[\PP{\hT(\cA,\cD_\ell,\cD_u) = 1, q_{c(\cD^{(1)}_\ell)} = \cdots = q_{c(\cD^{(\hr)}_\ell)} = 0} \geq \PP{\hT(\cA,\cD_\ell,\cD_u) = 1}\cdot (1-\rho)^{\frac{B_{\textnormal{train}}}{n}}.\]
Combining all our calculations, and rearranging terms,
\[\PP{\hT(\cA,\cD_\ell,\cD_u) = 1} \leq (1-\rho)^{-\frac{B_{\textnormal{train}}}{n}}\left[\alpha + \PP{\sum_{\mathbf{i}\in\mathbf{I}_{n,M}} (1-q_{\mathbf{i}})\cdot p_{\mathbf{i}} \geq 1-\delta}\right].\]
Now for our final step, we need to bound this probability.  Then
\begin{multline*}\EE{\sum_{\mathbf{i}\in\mathbf{I}_{n,M}}  (1-q_{\mathbf{i}})\cdot p_{\mathbf{i}}}
=(1-\rho) \sum_{\mathbf{i}\in\mathbf{I}_{n,M}} p_{\mathbf{i}} \\= (1-\rho)\PP{|[\cA(\cD_n,\xi)](X_{n+1}) - [\cA(\cD_{n-1},\xi)](X_{n+1})|\leq \epsilon} = (1-\rho)(1-\delta^*_\epsilon),
\end{multline*}
by definition of $p_{\mathbf{i}}$.
And, by Hoeffding's inequality,
\[\PP{\sum_{\mathbf{i}\in\mathbf{I}_{n,M}}  (1-q_{\mathbf{i}})\cdot p_{\mathbf{i}} \geq (1-\rho)(1-\delta^*_\epsilon) + t}\leq \exp\left\{-\frac{2t^2}{\sum_{\mathbf{i}\in\mathbf{I}_{n,M}}p_{\mathbf{i}}^2}\right\}\]
for all $t>0$, since $\sum_{\mathbf{i}\in\mathbf{I}_{n,M}}  (1-q_{\mathbf{i}})\cdot p_{\mathbf{i}}$ is a sum of independent terms, with the term indexed $\mathbf{i}$ lying in $p_{\mathbf{i}}$. 
Finally, $\sum_{\mathbf{i}\in\mathbf{I}_{n,M}}p_{\mathbf{i}}^2\leq \|p\|_\infty\cdot \sum_{\mathbf{i}\in\mathbf{I}_{n,M}}p_{\mathbf{i}} \leq \|p\|_\infty$. Combining with the calculation above, then,
\[\PP{\sum_{\mathbf{i}\in\mathbf{I}_{n,M}}  (1-q_{\mathbf{i}})\cdot p_{\mathbf{i}} \geq (1-\rho)(1-\delta^*_\epsilon) + t}\leq \exp\left\{-\frac{2t^2}{\|p\|_\infty}\right\},\]
and therefore, if we choose $\rho > \frac{\delta - \delta^*_\epsilon}{1-\delta^*_\epsilon}$,
\[\PP{\hT(\cA,\cD_\ell,\cD_u) = 1} \leq (1-\rho)^{-\frac{B_{\textnormal{train}}}{n}}\left[\alpha + \exp\left\{-\frac{2( \rho(1-\delta^*_\epsilon)- (\delta -\delta^*_\epsilon))^2}{\|p\|_\infty}\right\}\right].\]
Setting $\rho = \frac{\delta - \delta^*_\epsilon + \frac{1}{n}}{1-\delta^*_\epsilon}$,
\[\PP{\hT(\cA,\cD_\ell,\cD_u) = 1} \leq \left(\frac{1-\delta^*_\epsilon}{1 - \delta - \frac{1}{n}}\right)^{\frac{B_{\textnormal{train}}}{n}}\left[\alpha + \exp\left\{-\frac{2}{n^2\|p\|_\infty}\right\}\right].\]
Finally, since we have assumed $\sup_{(x,y)\in\cX\times\cY}P(\{(x,y)\}) \leq \gamma < 0.2$, we can take $M=6$ in the following lemma so that $\|p\|_\infty = \mathcal{O}(n^{-2.5})$:
\begin{Lemma}\label{lem:data_counts}
    Fix any $M\geq 2$. Suppose $\sup_{(x,y)\in\cX\times\cY}P(\{(x,y)\}) \leq \gamma$ for some constant $0\leq \gamma < \frac{1}{M-1}$. 
    Then there exists a partition
    $\cX\times\cY = \cup_{m=1}^M C_m$
    such that, for a dataset $\cD_n = ((X_1,Y_1),\dots,(X_n,Y_n))\sim P^n$,
    \[\max_{\mathbf{i}\in\mathbf{I}_{n,M}}\PP{c(\cD_n) = \mathbf{i}} \leq \frac{C_{M,\gamma}}{n^{(M-1)/2}},\]
    where $C_{M,\gamma}$ is a constant depending only on $M$ and $\gamma$.
\end{Lemma}
This completes the proof.

\section{Proof of Proposition \ref{prop:test-power} }
The proof is similar to the proof of Theorem 1 in \cite{kim2021black}, we include it here for completeness.
	First, we have $B \sim \textnormal{Binomial}(\lfloor \kappa \rfloor, \delta^*_{\epsilon }  )$, where $\delta^*_{\epsilon }$ is defined in \eqref{def:delta_star}. Now, we verify the validity of this test. Fix any $(\cA, P, n)$ such that $\delta^*_{\epsilon } > \delta$, then 
	\begin{equation*}
		\begin{split}
			\bbP(\hT_{\textnormal{Binom}}(\cA,\cD_{\ell}, \cD_u) = 1) &= \bbP(B < k_*) + a_* \bbP(B = k_*) \\
			& = (1 - a_*) \bbP(B < k_*) + a^*\bbP(B \leq k_*) \\
			& \overset{(a)}\leq (1 - a_*) \bbP(\textnormal{Binomial}(\lfloor \kappa \rfloor, \delta ) < k_*) + a_* \bbP(\textnormal{Binomial}(\lfloor \kappa \rfloor, \delta ) \leq k_* ) \\
			& = \bbP(\textnormal{Binomial}(\lfloor \kappa \rfloor, \delta ) < k_*) + a_* \bbP(\textnormal{Binomial}(\lfloor \kappa \rfloor, \delta ) = k_* ) \\
			& \overset{(b)}= \alpha,
		\end{split}
	\end{equation*} where (a) is because $\textnormal{Binomial}(\lfloor \kappa \rfloor, \delta^*_{\epsilon}  )$ stochastically dominates $\textnormal{Binomial}(\lfloor \kappa \rfloor, \delta  )$ and (b) is by the definition of $k_*$ and $a_*$.
	
	Next, let us compute the power when $\delta < 1 - \alpha^{1/\lfloor \kappa \rfloor }$. Notice that now $\bbP( \textnormal{Binomial}(\lfloor \kappa \rfloor, \delta  )  = 0 ) = (1-\delta)^{ \lfloor \kappa \rfloor } > \alpha$, so it implies that
	\[k_* = 0, \ a_* = \frac{\alpha}{\bbP( \textnormal{Binomial}(\lfloor \kappa \rfloor, \delta  )  = 0 )}.\]
Therefore, by the construction of the Binomial test $\hT$, we have
\begin{equation*}
	\begin{split}
		\bbP(\hT_{\textnormal{Binom}}(\cA,\cD_{\ell}, \cD_u) = 1)  = a_* \bbP(B = 0) = \frac{\alpha}{(1 -\delta)^{\lfloor \kappa \rfloor}} \cdot (1 - \delta^*_{\epsilon })^{\lfloor \kappa \rfloor} = \alpha \left( \frac{1 - \delta^*_{ \epsilon }}{1 - \delta} \right)^{\lfloor \kappa \rfloor}.
	\end{split}
\end{equation*} This finishes the proof of this proposition by recalling the definition of $\lfloor \kappa \rfloor$ in \eqref{def:kappa}.

\section{Proof of Theorem \ref{thm:limits_evaluate-black-box-model}} \label{sec:proof-theorem2}
The first three terms in the power bound are the same as the ones appearing
in Theorem \ref{thm:limits_evaluate-random}. Since this new result, Theorem \ref{thm:limits_evaluate-black-box-model}, places strictly more constraints
on the test $\hT$ (i.e., due to the model evaluation budget $B_{\textnormal{train}}$), 
these three bounds on power hold in this setting as well.
To complete the proof, then, we only need to establish the last bound, namely,
\[
		\bbP \{ \hT = 1 \} \leq \alpha  \cdot \frac{\left( \frac{1 - \delta^*_{\epsilon} }{1 -\delta }\right)^{ \frac{N_{\ell} + N_u}{n+1} }}{1 - \frac{B_{\textnormal{train}}+B_{\textnormal{eval}}}{|\cX|}\wedge 1 } .
\]
The proof is similar to the one in Section \ref{sec:X-space-bound-pf}, except now we define the event $\widetilde\cE_x$ as $x \notin \cD_{\ell}^{(1)}\cup \cX^{(1)} \cup \dots \cup \cD_{\ell}^{(\hr)} \cup \cX^{(\hr)}$ since we are now interested in the role of limited $X$ data in terms of model evaluation as well as model training.
As before, we will need a lemma:
\begin{Lemma}\label{lem:X_eval}
In the setting of Theorem~\ref{thm:limits_evaluate-black-box-model},
for any $x\in\cX$,
\[\bbP\{\hT = 1; \widetilde\cE_x\} \leq \alpha\left(\frac{1-\delta^*_\epsilon}{1-\delta}\right)^{\frac{N_\ell+N_u}{n+1}}.\]
\end{Lemma}
Now we prove that this lemma implies the desired bound on power. This calculation is similar as for the transparent-model setting (as in Section~\ref{sec:Y-space-pf}).

First, we can assume $|\cX|> B_{\textnormal{train}} + B_{\textnormal{eval}}$,
since otherwise the bound is trivial. 
Fix any integer $M$ with $B_{\textnormal{train}} + B_{\textnormal{eval}}<M\leq |\cX|$  (note that $|\cX|$ may be finite or infinite),
and let $x_1,\dots,x_M\in\cX$ be distinct.
We then have \begin{multline*}
\sum_{i=1}^M \bbP\{\hT = 1; \widetilde\cE_{x_i}\}
= \bbE\left[\indi\{\hT=1\} \cdot \sum_{i=1}^M \indi\{\widetilde\cE_{x_i}\}\right]\\
\geq \bbE\left[\indi\{\hT=1\} \cdot (M - B_{ \textnormal{train} } - B_{\textnormal{eval}})\right] = \bbP\{\hT=1\} \cdot  (M - B_{ \textnormal{train} } - B_{\textnormal{eval}}),
\end{multline*}
where the inequality holds since, due to the computational constraint,
$x\in\cD^{(1)}_\ell \cup\cX^{(1)} \cup\dots\cup\cD^{(\hr)}_\ell\cup\cX^{(\hr)}$ can hold for at most $B_{\textnormal{train}}+B_{\textnormal{eval}}$ many values $x\in\cX$.
Therefore,
\[\bbP\{\hT=1\}  \leq \frac{\sum_{i=1}^M \bbP\{\hT = 1; \widetilde\cE_{x_i}\}}{M - B_{ \textnormal{train} }-B_{\textnormal{eval}}} 
\leq\frac{M\cdot  \alpha\left(\frac{1-\delta^*_\epsilon}{1-\delta}\right)^{\frac{N_\ell+N_u}{n+1}}}{M - B_{ \textnormal{train} }-B_{\textnormal{eval}}} ,\]
where the last step holds by Lemma~\ref{lem:X_eval}.
Since this holds for any $M$ with $B_{\textnormal{train}}-B_{\textnormal{eval}}<M\leq |\cX|$, as in Section~\ref{sec:Y-space-pf}
we can take $M=|\cX|$ to complete the proof for the case $|\cX|<\infty$, or $M\to\infty$ for the case $|\cX|=\infty$.

\begin{proof}[Proof of Lemma~\ref{lem:X_eval}]
The proof of this lemma is similar to that of Lemma~\ref{lem:X} but now we must also account for
limits on model evaluation.
Fix any $c\in[0,1]$.
We define $P'$ exactly as in the proof of Lemma~\ref{lem:X}, i.e., $P' = c \cdot (\delta_x\times P_Y) + (1-c)\cdot P$, but the modified algorithm is defined differently. We define $\cA'$ by specifying, for any training set $\cD$ and random seed $\xi$, the value of the fitted model when evaluated at a test point $x'$:
\begin{equation}\label{def:A-prime-x_eval}
	\cA'(\cD;\xi)(x') = \begin{cases}\cA_{1}(\cD; \xi)(x') & \textnormal{ if }|\cD| = n, \textnormal{ and either $x\in\cD$ or $x'=x$},\\
 \cA(\cD;\xi)(x'), & \textnormal{ otherwise} ,\end{cases}
\end{equation}
 	where $\cA_{1}$ is defined as in the proof of Lemma~\ref{lem:Btrain}.
	With this construction, 
 we then have
	\[ \bbP_{P'} \left\{ |\cA'(\cD_n; \xi)(X_{n+1}) - \cA'(\cD_{n-1}; \xi)(X_{n+1}) | > \epsilon  \right\} \geq \left[1 - (1-c)^{n+1}\right] + (1-c)^{n+1}\delta^*_\epsilon,\]
where, similarly to the analogous calculation in the proofs of Lemmas~\ref{lem:Y} and~\ref{lem:X}, the term in square brackets arises from the event that at least one of the $n+1$ data points (the training set $\cD_n$ and the test point $X_{n+1}$) is drawn
from the first component in the mixture distribution $P'$, i.e., from $\delta_x\times P_Y$,
while the remaining term arises from the event that all $n+1$ data points
are drawn from the second component, $P$.
In particular, we have
\[ \bbP_{P'} \left\{ |\cA'(\cD_n; \xi)(X_{n+1}) - \cA'(\cD_{n-1}; \xi)(X_{n+1}) | > \epsilon  \right\}>\delta\]
as long as $c > 1 -\left(\frac{1-\delta}{1-\delta^*_\epsilon}\right)^{1/(n+1)}$.
 
Next, since $\cA(\cD^{(r)}_\ell;\xi^{(r)})(x') = \cA'(\cD^{(r)}_\ell;\xi^{(r)})(x')$ for all $x'\in\cX^{(r)}$, 
as long as $x\not\in \cD^{(r)}_\ell$ and $x'\neq x$, this means that running the test $\hT$ with $\cA'$ or with $\cA$ returns identical
answers on the event $\widetilde\cE_x$.
We therefore have \begin{multline*}\bbP_P\{\hT(\cA,\cD_\ell,\cD_u)=1;\widetilde\cE_x\} = 
\bbP_P\{\hT(\cA',\cD_\ell,\cD_u)=1; \widetilde\cE_x\} \leq \bbP_P\{\hT(\cA',\cD_\ell,\cD_u)=1\} \\
\leq (1-c)^{- (N_\ell+N_u)}\cdot \bbP_{P'}\{\hT(\cA',\cD_\ell,\cD_u)=1\} 
\leq \alpha (1-c)^{-(N_\ell+N_u)},\end{multline*}
where the last step holds by validity of $\hT$~\eqref{eqn:validity}, since $(\cA',P',n)$ is not $(\epsilon,\delta)$-stable.
We have established that this is true for any $c\in[0,1]$ with $c > 1 -\left(\frac{1-\delta}{1-\delta^*_\epsilon}\right)^{1/(n+1)}$, and so we have
\[\bbP_P\{\hT(\cA,\cD_\ell,\cD_u)=1; \widetilde\cE_x\} 
\leq \alpha \left(\frac{1-\delta^*_\epsilon}{1-\delta}\right)^{\frac{N_\ell+N_u}{n+1}}.\]
\end{proof}

\section{Additional proofs and lemmas}

\begin{proof}[Proof of Lemma~\ref{lem:data_counts}]
     First by Lemma~\ref{lem:partition}, we can find some partition $\cX\times\cY = C_1 \cup \dots \cup C_M$, such that
    \[\min_{m=1,\dots,M} P(C_m) \geq \min\left\{\frac{1}{2M-1},1-(M-1)\gamma\right\}.\]
    Then since $\cD_n$ is sampled from $P^n$, we have
    \[c(\cD_n)\sim\textnormal{Multinomial}(n,(P(C_1),\dots,P(C_M)).\]
    Applying Lemma~\ref{lem:multinom}, then,
    \[\max_{\mathbf{i}} \PP{c(\cD_n) =\mathbf{i}}\leq  \frac{C'_M}{\sqrt{  n^{M-1} \cdot \prod_{m=1}^M P(C_m)}}, \] where $C_M'$ is a constant depends only on $M$.
    Combining all these calculations proves the desired bound.
\end{proof}

\begin{Lemma}\label{lem:multinom}
    Let $M\geq 2$ and let $\mathbf{q}=(q_1,\dots,q_M)\in\Delta_M\subseteq\R^M$, where $\Delta_M$ is the probability simplex. Let $p_{n,\mathbf{q}}$ be the PMF of the Multinomial$(n,\mathbf{q})$ distribution, and assume $q_m>0$ for all $m$. Then for all $n\geq 1$, 
    \[\max_{\mathbf{i}}p_{n,\mathbf{q}}(\mathbf{i}) \leq  \frac{C_M}{\sqrt{n^{M-1} \cdot \prod_{m=1}^M q_m}},\]
    where the maximum is taken over all $\mathbf{i} = (i_1,\dots,i_M)$ where the $i_m$'s are nonnegative integers with $\sum_m i_m = n$, and where the constant $C_M$ depends only on $M$.
\end{Lemma}
\begin{proof}[Proof of Lemma~\ref{lem:multinom}]
We will prove the result by induction. First consider the case $M=2$---then the Multinomial$(n,\mathbf{q})$ distribution is equivalent to the Binomial$(n,q_1)$ distribution. We will use a CLT argument. Specifically, by the Berry--Esseen theorem, 
\[\sup_{x\in\R}\left|\PP{\textnormal{Binomial}(n,q_1)\leq x} - \Phi\left(\frac{x - n\mu}{\sqrt{n}\sigma}\right)\right|\leq 0.56 \frac{\rho}{\sqrt{n}\sigma^3}\]
where $\mu = q_1$, $\sigma^2=q_1(1-q_1)$, and $\rho = q_1(1-q_1)(q_1^2+(1-q_1)^2\leq \sigma^2$ are calculated from the moments of the Bernoulli. Therefore,
\[\sup_{x\in\R}\left|\PP{\textnormal{Binomial}(n,q_1)\leq x} - \Phi\left(\frac{x - nq_1}{\sqrt{nq_1(1-q_1)}}\right)\right|\leq  \frac{0.56}{\sqrt{nq_1(1-q_1)}}.\]
Then for any $\mathbf{i}=(i_1,i_2)$
\begin{multline*}p_{n,\mathbf{q}}(\mathbf{i}) = \PP{\textnormal{Binomial}(n,q_1)=i_1}
\leq \PP{\textnormal{Binomial}(n,q_1)\leq i_1} - \PP{\textnormal{Binomial}(n,q_1)\leq i_1-\epsilon} \\\leq \left(\Phi\left(\frac{i_1 - nq_1}{\sqrt{nq_1(1-q_1)}}\right) - \Phi\left(\frac{i_1-\epsilon - nq_1}{\sqrt{nq_1(1-q_1)}}\right)\right) + \frac{1.12}{\sqrt{nq_1(1-q_1)}}. \end{multline*}
Taking $\epsilon\to0$ proves that
\[p_{n,\mathbf{q}}(\mathbf{i}) \leq  \frac{1.12}{\sqrt{nq_1(1-q_1)}} =  \frac{1.12}{\sqrt{nq_1q_2}}.\]
Setting $C_2 = 1.12$, this completes the proof for the case $M=2$.

Now we proceed by induction. Suppose $M\geq 3$. Let
\[\mathbf{q}' = (q_1,\dots,q_{M-2},q_{M-1}+q_M)\in\Delta_{M-1}\]
and
\[\mathbf{q}'' = \left(\frac{q_{M-1}}{q_{M-1}+q_M}, \frac{q_M}{q_{M-1}+q_M}\right)\in\Delta_2.\]
Then we can decompose the Multinomial$(n,\mathbf{q})$ PMF as follows:
\[p_{n,\mathbf{q}}(\mathbf{i}) = p_{n,\mathbf{q}'}((i_1,\dots,i_{M-2},i_{M-1}+i_M)) \cdot p_{i_{M-1}+i_M,\mathbf{q}''}((i_{M-1},i_M)). \]
The intuition is that first we sample over $M-1$ categories (by merging bins $M-1$ and $M$ into a single bin), then we sample from a Binomial to separate this last merged bin into its two components.
By induction,
\[p_{n,\mathbf{q}'}((i_1,\dots,i_{M-2},i_{M-1}+i_M)) \leq \frac{C_{M-1}}{\sqrt{n^{M-2} \cdot\prod_{m=1}^{M-2}q_m \cdot (q_{M-1}+q_M)}}\]
and
\[p_{i_{M-1}+i_M,\mathbf{q}''}((i_{M-1},i_M)) \leq \frac{C_2}{\sqrt{n \cdot \left(\frac{q_{M-1}}{q_{M-1}+q_M}\right) \cdot \left(\frac{q_M}{q_{M-1}+q_M}\right)}} \leq \frac{C_2 \sqrt{q_{M-1}+q_M}}{\sqrt{nq_{M-1}q_M}}.\]
Combining these calculations, and setting $C_M=C_{M-1}C_2$, we have proved the desired bound.
\end{proof}

\begin{Lemma}\label{lem:partition}
    Fix any $M\geq 2$. Let $P$ be a distribution on $\cZ$, with \[\sup_{z\in\cZ}P(\{z\}) \leq \gamma\] for some $\gamma\in[0,\frac{1}{M-1})$. Then there exists a partition $\cZ = C_1 \cup \dots \cup C_M$ such that \[\min_{m=1,\dots,M}P(C_m) \geq \min\left\{\frac{1}{2M-1}, 1 - (M-1)\gamma\right\}.\] 
\end{Lemma}
\begin{proof}[Proof of Lemma~\ref{lem:partition}] 
First, 
 apply Lemma~\ref{lem:partition_2} (with $\max\left\{\gamma, \frac{2}{2M-1}\right\}$ in place of $\gamma$), to find some $C_1\subseteq\cZ$ with
 \begin{equation}\label{eqn:induction_C1}\frac{\max\left\{\gamma, \frac{2}{2M-1}\right\}}{2}\leq P(C_1) \leq \max\left\{\gamma, \frac{2}{2M-1}\right\}.\end{equation}
In particular, this means
 \[P(C_1)\geq  \frac{1}{2M-1}.\]
If $M=2$, then we set $C_2 = \cZ \backslash C_1$, and we then have $P(C_2) = 1-P(C_1) \geq 1-\gamma = 1-(M-1)\gamma$, therefore the proof is complete.

Otherwise, if $M\geq 3$, we proceed by induction. Let $M'=M-1\geq 2$, let $\cZ' = \cZ \backslash C_1$, and let $P'$ be the distribution of $Z\sim P$ conditional on the event $Z\not\in C_1$ (i.e., $P'$ is supported on $\cZ'$). Then
\[\sup_{z\in\cZ'}P'(\{z\}) = \sup_{z\in\cZ'} \frac{P(\{z\})}{P(\cZ\backslash C_1)} \leq \frac{\gamma}{P(\cZ\backslash C_1)} \leq \frac{\max\left\{\gamma, \frac{2}{2M-1}\right\}}{P(\cZ\backslash C_1)} =: \gamma'.\]
Since $\gamma < \frac{1}{M-1}$ by assumption, therefore $\max\left\{\gamma, \frac{2}{2M-1}\right\}< \frac{1}{M-1}$. Since $P(\cZ\backslash C_1)\geq 1-\max\left\{\gamma, \frac{2}{2M-1}\right\}$ by~\eqref{eqn:induction_C1},
\[ \gamma' = \frac{\max\left\{\gamma, \frac{2}{2M-1}\right\}}{P(\cZ\backslash C_1)} \leq \frac{\max\left\{\gamma, \frac{2}{2M-1}\right\}}{1-\max\left\{\gamma, \frac{2}{2M-1}\right\}} < \frac{\frac{1}{M-1}}{1-\frac{1}{M-1}} = \frac{1}{M-2} = \frac{1}{M'-1}.\]

Then by induction, applying this lemma with $M',P',\cZ',\gamma'$ in place of $M,P,\cZ,\gamma$, we can construct a partition
$\cZ' = C'_1 \cup \dots \cup C'_{M'}$ such that 
\[\min_{m=1,\dots,M'} P'(C'_m) \geq  \min\left\{\frac{1}{2M'-1},1 - (M'-1)\gamma'\right\}.\]
Plugging in our definitions of $M',P',\gamma'$, and
rearranging terms,
\begin{align*}
    \min_{m=1,\dots,M-1} P(C'_m) 
    &\geq \min\left\{\frac{P(\cZ\backslash C_1)}{2M-3},P(\cZ\backslash C_1) - (M-2)\max\left\{\gamma, \frac{2}{2M-1}\right\}\right\}\\
    &\geq  \min\left\{\frac{1-\max\left\{\gamma, \frac{2}{2M-1}\right\}}{2M-3},1-(M-1)\max\left\{\gamma, \frac{2}{2M-1}\right\}\right\}\\
    &= 1-(M-1)\max\left\{\gamma, \frac{2}{2M-1}\right\}\\
    &=\min\left\{\frac{1}{2M-1},1-(M-1)\gamma\right\},
\end{align*}
where the second step holds by~\eqref{eqn:induction_C1}, and the third step holds since $\frac{1-t}{2M-3} \geq 1-(M-1)t$ for any $t\geq \frac{2}{2M-1}$.

Finally, define $C_m = C'_{m-1}$ for $m=2,\dots,M$. Then we have constructed the desired partition.
\end{proof}

\begin{Lemma}\label{lem:partition_2}
    Fix any $M\geq 2$. Let $P$ be a distribution on $\cZ$, with \[\sup_{z\in\cZ}P(\{z\}) \leq \gamma\] for some $\gamma \in [0,1]$. Then there exists some $C\subseteq \cZ$ such that 
    \[\frac{\gamma}{2} \leq P(C) \leq \gamma.\]
\end{Lemma}
\begin{proof}[Proof of Lemma~\ref{lem:partition_2}]
    Let $N\in\{0,1,2,\dots\}\cup\{+\infty\}$ be the number of point masses of $P$.
    First we decompose $P$ into discrete and nonatomic components,
    \[P = P_0 + P_1,\]
    where $P_1$ is a nonatomic measure on $\cZ$, while $P_0$ is a discrete measure supported on $\cZ_0\subseteq \cZ$. If $P_1(\cZ) \geq 1 - \frac{\gamma}{2}$, then by \citet[Proposition A.1]{dudley2011concrete}, we can find some $C'\subseteq\cZ$ with $P_1(C') = 1 - \frac{\gamma}{2}$. 
    Then defining 
    $C = (\cZ \backslash C') \cup \cZ_0$, we have \[P_1(C) = P_1(\cZ \backslash C') = P_1(\cZ) - P_1(C') ,\] where the first step holds since $P_1$ is nonatomic and so $P_1(\cZ_0)=0$. And,
    \[P_0(C) = P_0(C \cap \cZ_0) = P_0(\cZ_0) = P_0(\cZ) = 1-P_1(\cZ),\]
    since $C\supseteq\cZ_0$, and $P_0$ is supported on $\cZ_0$. Therefore,
    \[P(C) = P_0(C) + P_1(C) = 1 - P_1(C') = \frac{\gamma}{2}.\]
     This completes the proof for this case.

    If instead $P_1(\cZ) < 1 - \frac{\gamma}{2}$, then we have $P_0(\cZ)> \frac{\gamma}{2}$. Since $P_0$ is a discrete distribution, we can therefore find a finite set of distinct points $z_1,\dots,z_K\in\cZ$ such that $\sum_{k=1}^KP_0(\{z_k\}) >\frac{\gamma}{2}$. Without loss of generality assume that
    \[P(\{z_1\})\geq \dots \geq P(\{z_K\}).\]
    Since $\sum_{k=1}^KP(\{z_k\}) =\sum_{k=1}^KP_0(\{z_k\}) >\frac{\gamma}{2}$, the integer
    \[K_0 = \min\left\{k_0 : \sum_{k=1}^{k_0} P(\{z_k\}) \geq \frac{\gamma}{2}\right\}\]
    is well-defined (i.e., the set above is nonempty). 
    Now define $C = \{z_1,\dots,z_{K_0}\}$.  By definition of $K_0$ we have
    \[ P(C) =  \sum_{k=1}^{K_0} P(\{z_k\})\geq \frac{\gamma}{2}.\]
    To complete the proof we need to show that $P(C) \leq \gamma$. If $K_0=1$, then
    \[ P(C) = P(\{z_1\}) \leq \gamma, \]
    by assumption. If instead $K_0\geq 2$, then
    \[P(C) = \sum_{k=1}^{K_0}P(\{z_k\}) = P(\{z_{K_0}\}) + \sum_{k=1}^{K_0-1}P(\{z_k\})< \frac{\gamma}{2}+\frac{\gamma}{2} = \gamma,\]
    where the last step holds since $P(\{z_{K_0}\}) \leq P(\{z_1\}) \leq \sum_{k=1}^{K_0-1}P(\{z_k\})$, and we must have $\sum_{k=1}^{K_0-1}P(\{z_k\}) < \gamma/2$ (since if not, this would contradict the definition of $K_0$). This completes the proof.
\end{proof}

\end{document}